\documentclass{article}
\usepackage{iclr2027_conference,times}

\newif\ificlrpreprint
\newcommand{\iclrpreprintcopy}{\iclrpreprinttrue\iclrfinalcopy}
\iclrpreprintcopy    

\usepackage[utf8]{inputenc} 
\usepackage[T1]{fontenc}    
\usepackage{hyperref}       
\usepackage{url}            
\usepackage{booktabs}       
\usepackage{amsfonts}       
\usepackage{nicefrac}       
\usepackage{microtype}      
\usepackage[table]{xcolor}         
\usepackage{mathtools} 
\usepackage{amsmath}
\usepackage{amssymb}
\usepackage{amsthm}
\usepackage{wrapfig}
\usepackage{caption}
\usepackage{bm}        
\usepackage{xfrac}     
\usepackage[separate-uncertainty=true]{siunitx}
\usepackage{graphicx}
\usepackage{tikz}
\usetikzlibrary{arrows.meta,positioning,fit,calc,backgrounds,decorations.pathreplacing,shapes.geometric}
\usepackage{subcaption}
\usepackage{adjustbox}
\usepackage[most]{tcolorbox}
\usepackage{multicol}
\usepackage[inline]{enumitem}
\usepackage{makecell}
\usepackage{multirow}
\usepackage{thmtools}
\usepackage{thm-restate}
\usepackage{mdframed}
\definecolor{summarygray}{RGB}{248,248,248}
\definecolor{summaryblue}{RGB}{55,105,170}
\newmdenv[
    backgroundcolor=summarygray,
    linecolor=summaryblue,
    linewidth=0pt,
    topline=false,
    rightline=false,
    bottomline=false,
    leftline=true,
    leftmargin=0pt,
    rightmargin=0pt,
    innerleftmargin=10pt,
    innerrightmargin=10pt,
    innertopmargin=7pt,
    innerbottommargin=7pt,
    skipabove=8pt,
    skipbelow=8pt,
    linewidth=2.2pt
]{summarybox}

\usepackage{csquotes} 
\usepackage{pifont}   
\usepackage{etoc}     
\usepackage{todonotes}
\usepackage{ifthen}

\usepackage{hyperref} 
\usepackage{cleveref} 
\let\cite\citep
\theoremstyle{definition}

\theoremstyle{plain}
\newtheorem{theorem}{Theorem}
\newtheorem{corollary}{Corollary}
\newtheorem{lemma}{Lemma}
\newtheorem{proposition}{Proposition}
\newtheorem{proof-sketch}{Proof Sketch}
\newif\ifgapetheorembars
\makeatletter
\@ifpackageloaded{iclr2027_conference}{%
  \ificlrfinal\gapetheorembarstrue\fi
}{%
  \@ifpackagewith{neurips_2026}{final}{\gapetheorembarstrue}{}
  \@ifpackagewith{neurips_2026}{preprint}{\gapetheorembarstrue}{}
}
\@ifundefined{assumption}{}{}
\makeatother

\definecolor{proofbar}{RGB}{225,225,225}
\definecolor{theorembar}{RGB}{128,0,128}
\definecolor{propositionbar}{RGB}{0,100,0}
\definecolor{definitionbar}{RGB}{0,0,200}
\definecolor{lemmabar}{RGB}{255,140,0}
\definecolor{corollarybar}{RGB}{204,153,255}
\definecolor{remarkbar}{RGB}{0,139,139}
\definecolor{assumptionbar}{RGB}{178,34,34}

\ifgapetheorembars
\surroundwithmdframed[
  hidealllines=true,leftline=true,linecolor=proofbar,linewidth=2pt,
  innerleftmargin=6pt,innerrightmargin=0pt,innertopmargin=2pt,innerbottommargin=2pt
]{proof}

\surroundwithmdframed[
  hidealllines=true,leftline=true,linecolor=theorembar,linewidth=2pt,
  innerleftmargin=6pt,innerrightmargin=0pt,innertopmargin=2pt,innerbottommargin=2pt
]{theorem}
\surroundwithmdframed[
  hidealllines=true,leftline=true,linecolor=propositionbar,linewidth=2pt,
  innerleftmargin=6pt,innerrightmargin=0pt,innertopmargin=2pt,innerbottommargin=2pt
]{proposition}
\surroundwithmdframed[
  hidealllines=true,leftline=true,linecolor=remarkbar,linewidth=2pt,
  innerleftmargin=6pt,innerrightmargin=0pt,innertopmargin=2pt,innerbottommargin=2pt
]{remark}
\surroundwithmdframed[
  hidealllines=true,leftline=true,linecolor=definitionbar,linewidth=2pt,
  innerleftmargin=6pt,innerrightmargin=0pt,innertopmargin=2pt,innerbottommargin=2pt
]{definition}
\surroundwithmdframed[
  hidealllines=true,leftline=true,linecolor=lemmabar,linewidth=2pt,
  innerleftmargin=6pt,innerrightmargin=0pt,innertopmargin=2pt,innerbottommargin=2pt
]{lemma}
\surroundwithmdframed[
  hidealllines=true,leftline=true,linecolor=corollarybar,linewidth=2pt,
  innerleftmargin=6pt,innerrightmargin=0pt,innertopmargin=2pt,innerbottommargin=2pt
]{corollary}

\surroundwithmdframed[
  hidealllines=true,leftline=true,linecolor=proofbar,linewidth=2pt,
  innerleftmargin=6pt,innerrightmargin=0pt,innertopmargin=2pt,innerbottommargin=2pt
]{proof-sketch}
\surroundwithmdframed[
  hidealllines=true,leftline=true,linecolor=assumptionbar,linewidth=2pt,
  innerleftmargin=6pt,innerrightmargin=0pt,innertopmargin=2pt,innerbottommargin=2pt
]{assumption}

\fi

\newenvironment{gapestatementframe}[1]{%
  \ifgapetheorembars
    \begin{mdframed}[
      hidealllines=true,leftline=true,linecolor=#1,linewidth=2pt,
      innerleftmargin=6pt,innerrightmargin=0pt,
      innertopmargin=2pt,innerbottommargin=2pt
    ]
  \else
    \par\addvspace{\topsep}\noindent\ignorespaces
  \fi
}{%
  \ifgapetheorembars
    \end{mdframed}
  \else
    \par\addvspace{\topsep}
  \fi
}

\title{Remember to Forget: \\Gated Adaptive Positional Encoding}

\author{%
  \parbox[t]{\dimexpr\textwidth-2\tabcolsep\relax}{%
    \raggedright\bfseries
    Riccardo Ali$^{1}$\thanks{Equal contribution. Contacts:
      Riccardo Ali: \href{mailto:rma55@cam.ac.uk}{\nolinkurl{rma55@cam.ac.uk}};
      Alessio Borgi: \href{mailto:alessio.borgi@uniroma1.it}{\nolinkurl{alessio.borgi@uniroma1.it}},
      \href{mailto:ab3352@cam.ac.uk}{\nolinkurl{ab3352@cam.ac.uk}};
      Mario Severino: \href{mailto:mario.severino@phd.unipd.it}{\nolinkurl{mario.severino@phd.unipd.it}};
      Christopher Irwin: \href{mailto:cli23@cam.ac.uk}{\nolinkurl{cli23@cam.ac.uk}}.}
    \quad Alessio Borgi$^{1,3*}$ \quad Mario Severino$^{1,2*}$\\
    Pietro Li\`o$^{1}$ \quad Christopher Irwin$^{1*}$\\[6pt]
    \normalfont\small
    $^{1}$Department of Computer Science and Technology, University of Cambridge\\
    $^{2}$Department of Information Engineering, University of Padua\\
    $^{3}$Department of Computer, Control and Management Engineering, Sapienza University of Rome
  }%
}

\begin{document}
\maketitle
\ificlrpreprint
  \lhead{Preprint}
\fi
\begin{abstract}
Rotary Positional Encoding (RoPE) is widely used in modern large language models. However, when sequences are extended beyond the range seen during training, rotary phases can enter out-of-distribution regimes, leading to spurious long-range alignments, diffuse attention, and degraded retrieval. Existing remedies only partially address these failures, as they often trade local positional resolution for long-context stability. We propose \textbf{GAPE} (\textbf{G}ated \textbf{A}daptive \textbf{P}ositional \textbf{E}ncoding), a drop-in augmentation for positional encodings that introduces a content-aware bias directly into the attention logits while preserving the rotary geometry. GAPE decouples distance-based suppression from token importance through a query-dependent gate that contracts irrelevant context and a key-dependent gate that preserves salient distant tokens. We show that weakly protected distant context is exponentially attenuated as a function of the query gate, while selected keys can remain accessible through landmark protection. We further show that GAPE can be implemented within standard scaled dot-product attention. Empirically, GAPE improves long-context robustness across controlled retrieval and language-modeling experiments, extrapolating up to $8\times$ the training length. We further retrofit GAPE into a pretrained 7B model, maintaining performance on standard benchmarks and improving performance at the longest evaluated context. These results support adaptive context suppression as a complement to positional representation for long-context generalization.

\end{abstract}

\vspace{-0.3cm}
\section{Introduction}
\label{sec:introduction}

Transformers \citep{vaswani2017attention} have become the \emph{de facto} 
backbone of modern machine learning, largely because self-attention provides a flexible and scalable mechanism for sequence modeling. The attention mechanism on its own does not encode token order: it is intrinsically permutation-equivariant, so the architecture requires an explicit \emph{positional encoding} (PE) to represent the sequence structure. One of the most widely adopted encodings, especially in Large Language Models (LLMs), is Rotary Positional Embedding (RoPE) \citep{su2024roformer}, used in popular models such as LLaMA~3 \citep{grattafiori2024llama3herdmodels} and Gemma~\citep{gemmateam2024gemmaopenmodelsbased}. RoPE encodes relative position via orthogonal rotations applied to queries and keys before the attention operation, thereby providing a geometric structure that effectively models local dependencies. Despite widespread adoption of RoPE, its extrapolation behavior beyond the training horizon remains poorly understood and highly sensitive. Although earlier literature proposed RoPE variants grounded in the paradigm of scaling the base wavelength $\theta$ to extend a model's context window \cite{men2024base, peng2023yarn}, our work builds upon recent structural analyses demonstrating that this approach is fundamentally constrained by an \emph{interpolation-extrapolation} deadlock \cite{okaprobing, liu2023scaling, xu2024base}. 
As these studies establish, inflating $\theta$ delays out-of-distribution (OOD) phase explosion and preserves local interpolation, but causes lower-frequency channels to collapse toward identity maps. Conversely, shrinking $\theta$ smooths extrapolation curves by accelerating rotations, but undermines the long-range resolution necessary for reliable retrieval. Partial-RoPE variants such as $p$-RoPE \citep{barbero2024round} mitigate this trade-off by leaving a subset of channels unrotated; however, this mechanism is effectively equivalent to increasing the wavelengths and therefore remains exposed to the same underlying instabilities. We suggest that these limitations reflect a broader structural constraint: RoPE-like multiplicative encodings preserve precise local relative geometry, but lack a native mechanism to explicitly attenuate irrelevant long-range context.
Conversely, additive encodings such as ALiBi~\citep{press2021train} stabilize extended contexts through explicit distance penalties, but impose rigid suppression that can erase useful distant signals. Robust long-context modeling, therefore, requires both capabilities: high-resolution local structure and stable, selective attenuation of irrelevant distant information. 

To address this, we introduce \textbf{GAPE} (\textbf{G}ated \textbf{A}daptive \textbf{P}ositional \textbf{E}ncoding). Rather than forcing a base-scaling tradeoff or imposing a fixed distance decay, GAPE augments the native rotary geometry of RoPE with a data-dependent content-aware mask over the attention logits. Concretely, GAPE introduces two adaptive routing signals: a \emph{query gate} that allows each query token to modulate how strongly its context is contracted, and a \emph{key landmark} that controls how strongly each key is protected from distance-based penalization. Together, these components define a positional mask that selectively suppresses irrelevant background context while preserving  distant useful information, without modifying the underlying rotary geometry. In this way, GAPE complements the positional geometry provided by RoPE with an adaptive mechanism that determines how much context to suppress and which distant tokens should remain accessible.

\textbf{Contributions.} 
We validate our framework both theoretically and empirically. Our core contributions can be summarized as follows:
\textbf{(i)} We introduce \textsc{GAPE}, a drop-in augmentation for positional encodings that contracts weakly protected distant context while preserving selected keys, with a FlashAttention-compatible implementation.
\textbf{(ii)} We show that \textsc{GAPE} provides a mechanism for suppressing spurious OOD long-range alignments while preserving distant relevant tokens, and validate the resulting robustness in controlled Needle-in-a-Haystack retrieval up to $8\times$ the training context length.
\textbf{(iii)} We show that \textsc{GAPE} improves long-context robustness in language modeling and can be retrofitted into a pretrained 7B model without degrading standard capabilities, while improving performance at the longest evaluated context.

\vspace{-0.3cm}
\section{Background and Related Works}
\vspace{-0.2cm}
\label{sec:related_work}
\noindent\textbf{Positional encodings in Transformers.}
Positional encodings inject sequence order into permutation-equivariant Transformers, either through absolute coordinates added to token embeddings 
\citep{vaswani2017attention, devlin2019bert} or relative-position injected directly into attention logits \citep{shaw2018self, dai2019transformer, raffel2020exploring}. Modern LLMs 
predominantly use RoPE \citep{su2024roformer}, 
which encodes relative position via orthogonal rotations applied to queries and keys. RoPE decomposes queries and keys into two-dimensional chunks, rotating each at a different frequency $g_k \in G$, ranging from $g_1 = 1$ radian per token (highest frequency) to $g_{d/2} \approx 1/\theta$ radians per token (lowest), where $\theta$ is 
the base wavelength, defaulting to $10{,}000$ \citep{su2024roformer}.

\noindent\textbf{Long-context extrapolation and the base-scaling deadlock.}
A natural response to RoPE's extrapolation failures is to scale $\theta$. Position Interpolation \citep{chen2023extending}, 
YaRN \citep{peng2023yarn}, and LongRoPE \citep{ding2024longrope} remap rotary frequencies to reduce OOD phase angles at extended lengths. However, recent theoretical analyses reveal that this exposes an 
\emph{interpolation-extrapolation deadlock} \citep{okafrequency, liu2023scaling, xu2024base}: shrinking $\theta$ smooths extrapolation but harms long-range semantic discrimination, while inflating $\theta$ preserves 
local interpolation but devolves low-frequency channels into near-identity maps, ultimately colliding with floating-point precision limits \citep{liu2023scaling, xu2024base}. Concurrently, \citet{barbero2024round} and \citet{wertheimer2026frayed} demonstrate empirically and theoretically that low-frequency RoPE channels 
act as semantic carriers but are fragile at extended lengths, and that neither NoPE nor RoPE alone is sufficient: the two have complementary strengths that no single base-scaling strategy can reconcile. Importantly, these approaches act directly on the rotary geometry used at extended positions, whereas GAPE leaves that geometry unchanged and introduces a separate structural routing mechanism. This makes rotary frequency remapping and GAPE's structural routing complementary rather than mutually exclusive.

\noindent\textbf{Partial-RoPE and frequency truncation.}
In light of these geometric constraints, partial-RoPE variants restrict rotations to a subset of channels. $p$-RoPE \citep{barbero2024round} removes the lowest-frequency rotations to preserve more stable semantic channels, and reports improved perplexity in 2B-parameter models. RoPE-ID \citep{wertheimer2026frayed} follows a related strategy by isolating the high-frequency components. However, recent probing reveals that the semantic bands where energy concentrates are not globally fixed; their locations vary significantly across heads, layers, and models, making any static frequency truncation threshold fundamentally erratic \cite{okafrequency, okaprobing}. Furthermore, while these approaches can improve stability, their effectiveness depends on how much of the rotary spectrum is retained: stronger truncation can weaken local positional structure, and even the remaining frequency-bearing channels can exhibit extrapolation instability if they do not undergo sufficient phase variation during training. At the same time, softmax dilution remains an independent issue, as distant background tokens can still accumulate mass at long context lengths, often requiring additional heuristics such as temperature scaling. GAPE is designed to complement these variants by adding an explicit, data-dependent mask on top of the rotary backbone, rather than relying on frequency truncation alone.

\noindent\textbf{Additive biases and distance-based attenuation.} A complementary family of methods injects position-dependent biases directly into attention logits. ALiBi \citep{press2021train} applies a linear distance penalty, while KERPLE \citep{chi2022kerple} and FIRE \citep{li2023functional} learn richer scalar functions of relative displacement. These approaches improve extrapolation stability by explicitly bounding softmax denominator growth. However, because they remain purely distance-driven, they impose a rigid recency prior 
that makes long-range retrieval highly sensitive, often irreversibly erasing distant semantically relevant tokens. GAPE's landmark mechanism replaces this fixed treatment with content-dependent protection, continuously reducing the distance penalty assigned to selected keys rather than uniformly suppressing distant context.

\noindent\textbf{Content-aware and adaptive attention.}
Close in spirit to GAPE is FoX \citep{lin2025forgettingtransformersoftmaxattention}, which introduces a data-dependent forgetting mechanism to modulate contextual boundaries. GAPE differs by explicitly separating how strongly a query contracts its context from how strongly individual keys are protected from that contraction, through the query gate $g_i$ and landmark gate $l_j$, respectively. FoX is also formulated as a replacement for rotary embeddings. Given the widespread use and strong empirical performance of RoPE, GAPE is instead designed as an augmentation: it preserves compatibility with existing RoPE configurations and with methods that modify their rotary geometry.

\begin{figure*}[t!]
\centering
\includegraphics[width=0.9\linewidth]{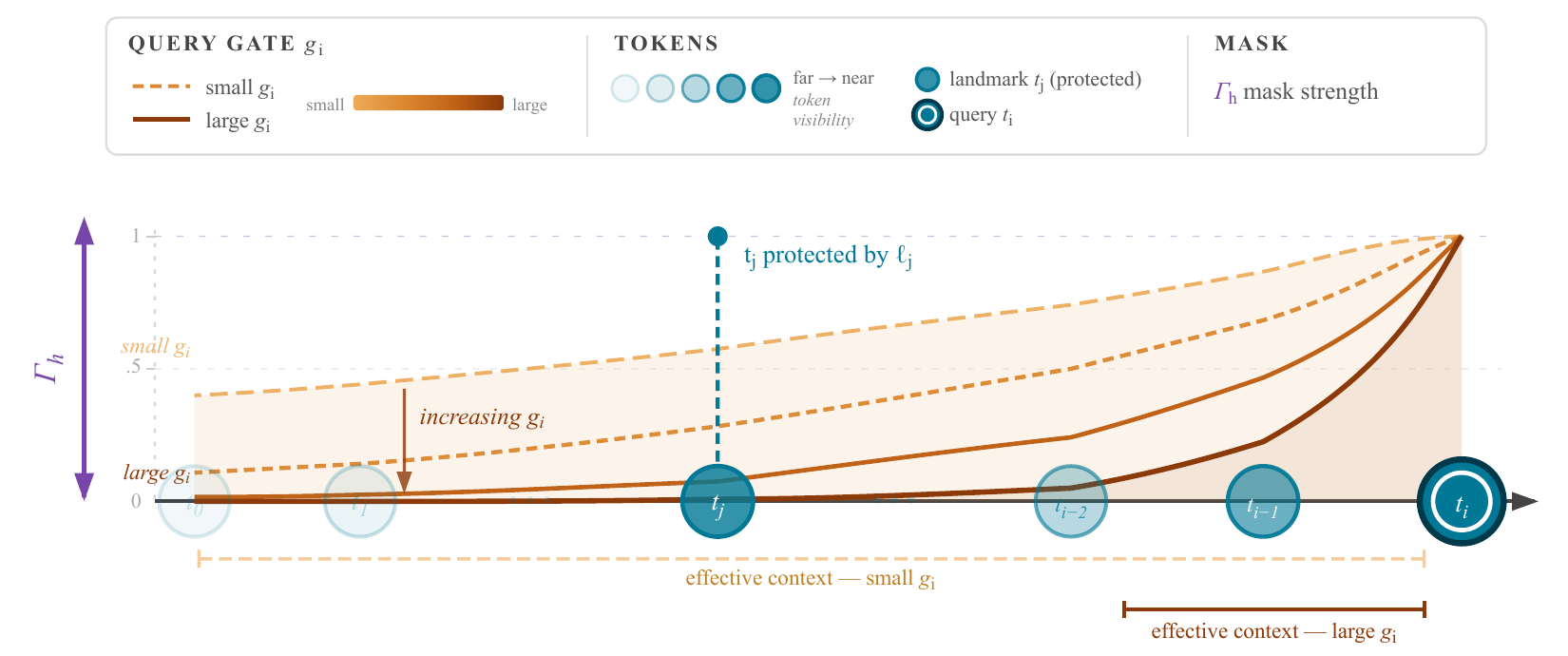}
\caption{
\emph{Gated Adaptive Positional Encoding.}
GAPE augments rotary attention with a content-dependent mask that separates context contraction from key protection. The query gate $g_i$ controls the effective positional horizon, the landmark gate $l_j$ continuously protects selected keys from distance-based suppression, and the head-wise amplitude $\Gamma_h$ calibrates the mechanism's overall strength. Together, these components suppress irrelevant long-range context while preserving distant information, complementing RoPE's positional geometry with adaptive structural routing.
}
\label{fig:gated_adaptive_pe}
\vspace{-5.5mm}
\end{figure*}

\vspace{-0.2cm}
\section{Gated Adaptive Positional Encoding}
\label{sec:GAPE} 
\vspace{-0.2cm}
Let $i$ be the current query position and $j \leq i$ be a causal key position, and let $R_{\Theta}(i-j)$ denote a generic RoPE-like positional map. The pre-softmax attention logit is:
\begin{equation}
    a_{i,j}
    =
    \underbrace{
    \frac{1}{\sqrt{d}}
    \mathbf{q}_i^\top R_{\Theta}(i-j)\mathbf{k}_j
    }_{s_{i,j}}
    +
    M_{i,j}
\end{equation}
where $s_{i,j}$ is the phase-modulated semantic similarity and $M_{i,j}$ is an additive structural bias applied on top of the multiplicative positional map, leaving the underlying rotary geometry unchanged. To make the role of this bias explicit, we first write its post-softmax-equivalent relative-distance form as $\widehat{M}_{i,j}=-\Gamma_h\,g_i\,(1-l_j)(i-j)/T$.

The mask separates four factors with distinct roles. The relative distance $(i-j)/T$ defines the positional envelope, $g_i$ controls query-dependent contraction, $(1-l_j)$ modulates how strongly that contraction applies to key $j$, and $\Gamma_h$ sets the head-wise strength. Since $l_j$ acts continuously, $(i-j)(1-l_j)$ can be interpreted as an effective structural distance. The linear term is not intended to approximate RoPE's oscillatory geometry, but to provide an independent monotone envelope that yields exponential attenuation after softmax, preserves translation invariance, and admits a separable SDPA realization.

\begin{itemize}[leftmargin=0.45cm, itemsep=1pt, topsep=2pt, parsep=0pt]
    \item \emph{Landmark gate} ($l_j \in [0,1]$): $l_j=\sigma(\mathbf{w}_l^\top\mathbf{k}_j)$ continuously controls key protection from distance-based suppression, with $l_j\to1$ approaching full protection.
    
    \item \emph{Query gate} ($g_i \in \mathbb{R}^{+}$): $g_i=\operatorname{Softplus}(\mathbf{w}_g^\top\mathbf{q}_i)$ controls contraction strength for query $i$, with larger values more strongly suppressing weakly protected distant keys.
    
    \item \emph{Head amplitude} ($\Gamma_h \in \mathbb{R}^{+}$): $\Gamma_h=\operatorname{Softplus}(\gamma_h)$ provides a persistent head-specific calibration of the structural pathway.
    
    \item \emph{Context normalizer} ($T$): Set to $T_{\mathrm{train}}$, it anchors the positional bias to the training regime, so distances beyond training are measured against the same fixed scale.
\end{itemize}

The two gates use different parameterizations to reflect their distinct roles. The sigmoid constrains the landmark score to $[0,1]$, matching its interpretation as a degree of protection, while Softplus keeps the query gate positive without imposing an upper bound on contraction strength. These parameterizations encode the intended ranges of the two gates while preserving smooth optimization dynamics.

Although $\widehat{M}_{i,j}$ as defined is not directly compatible with FlashAttention, the following equivalent formulation is:
\begin{equation}
    M_{i,j}
    =
    \frac{\Gamma_h\, g_i}{T}
    \Bigl[j(1 - l_j) + i\, l_j\Bigr]
    \label{eq:gape}
\end{equation}
As we show below, adding $M_{i,j}$ to the attention logits yields exactly the same post-softmax attention distribution as using $\widehat{M}_{i,j}$, while admitting a factorized implementation compatible with FlashAttention. We therefore \textbf{adopt $M_{i,j}$ as the GAPE mask} throughout the paper. The equivalence between the two formulations, together with the hardware-native realization of \textsc{GAPE}, is formalized as follows:\label{sec:hardware_factorization} 

\begin{proposition}
\label{prop:rank2_factorization}
The GAPE mask $M_{i,j}$ and the mask $\widehat{M}_{i,j}$ are equivalent post-softmax. Furthermore, $M_{i,j}$ can be computed within standard fused 
Scaled Dot-Product Attention (SDPA) without materializing an explicit $N \times N$ bias matrix. (Proof provided in Appendix~\ref{app:hardware_factorisation}.)
\end{proposition}

Concretely, we keep the head dimension fixed at $d$, allocating $d-2$ 
coordinates to the semantic component and reserving two for structural 
routing. In the factorized representation below, $\mathbf{q}_i$ and $\mathbf{k}_j$ denote the semantic coordinates after the chosen multiplicative positional map, so that their scaled inner product recovers $s_{i,j}$. The augmented query and key vectors $\tilde{\mathbf{q}}_i, 
\tilde{\mathbf{k}}_j \in \mathbb{R}^{d}$ are:
\begin{equation}
    \tilde{\mathbf{q}}_i =
    \begin{bmatrix}
    \mathbf{q}_i \\[2pt]
    \Gamma_h g_i \sqrt{d} \\
    \Gamma_h g_i \!\left(\tfrac{i}{T}\right)\!\sqrt{d}
    \end{bmatrix},
    \qquad
    \tilde{\mathbf{k}}_j =
    \begin{bmatrix}
    \mathbf{k}_j \\[2pt]
    \tfrac{j(1-l_j)}{T} \\
    l_j
    \end{bmatrix}
    \label{eq:augmented_qk}
\end{equation}
Under standard SDPA scaling, $\frac{1}{\sqrt{d}}\,\tilde{\mathbf{q}}_i^\top 
\tilde{\mathbf{k}}_j = s_{i,j} + M_{i,j} = a_{i,j}$. The structural bias is therefore realized exactly as a rank-at-most-two correction to the pre-softmax logits through the same matrix multiplication used for attention. GAPE requires no explicit $N\times N$ bias matrix or custom attention kernel and can be evaluated with FlashAttention-style fused implementations \citep{dao2022flashattentionfastmemoryefficientexact, 
dao2023flashattention2fasterattentionbetter, shah2024flashattention, 
zadouri2026flashattention4algorithmkernelpipelining}. The same factorization also preserves the translation-invariance property of relative positional encodings (\textit{Proof in Appendix~\ref{app:translation_invariance}}). We provide the corresponding computational and KV-cache analysis in Appendix~\ref{app:complexity}. We next characterize how the landmark and query gates shape the attention distribution, focusing on the behavior induced by their values rather than on any particular training dynamics.

\noindent\textbf{The landmark gate $l_j$ protects keys from distance-based suppression.}
Fixed recency biases such as ALiBi penalize keys solely as a function of their distance from the query, which can make relevant distant tokens effectively inaccessible. GAPE instead modulates this penalty continuously through $l_j$: intermediate values provide partial protection, while $l_j \to 1$ approaches the fully protected limit where the structural distance vanishes. We formalize this behavior as follows:

\begin{theorem}
\label{thm:bilateral_invariant}
Let $j_{\mathrm{distant}} \ll i$. If $l_{j_{\mathrm{distant}}}=1$, then $M_{i,i} = M_{i,j_{\mathrm{distant}}}$; more generally, $M_{i,j_{\mathrm{distant}}} \to M_{i,i}$ as $l_{j_{\mathrm{distant}}}\to1$.
Moreover, for any $a < b < i$ with $l_b < 1$, there exists a 
threshold $l_a^\ast \in (0,1)$ such that for all $l_a \in (l_a^\ast, 1]$: $M_{i,a} > M_{i,b}$. ({Proof provided in Appendix~\ref{app:selective_routing}.})
\end{theorem}

For the asymptotic analysis below, it is useful to distinguish keys approaching full landmark protection from the remaining context. We therefore define:
\begin{equation}
    P_i=\{j<i : l_j\to 1\} \cup\{i\}\;\text{ and } \;U_i=\{j<i : l_j\not\to 1 \}
\end{equation}
This partition is used to characterize the limiting behavior of the mask; for finite values, $l_j$ remains continuous and allows intermediate degrees of protection.

\noindent\textbf{The query gate $g_i$ controls the strength of context contraction.}
While $l_j$ controls how strongly each key is exposed to the structural penalty, $g_i$ sets its overall strength for the current query. Larger $g_i$ yields a narrower effective positional horizon by more strongly suppressing weakly protected distant context, whereas smaller $g_i$ preserves broader access. Their joint effect follows directly from the resulting attention weights.

\begin{theorem}
\label{thm:asymptotic_eradication}
Let $P_i$ and $U_i$ be the protected and unprotected sets defined above and let 
$S_{\max} := \max_m \|\mathbf q_i\|\,\|\mathbf k_m\|$. Then, the total attention mass on $U_i$
satisfies:
\begin{equation}
    \sum_{k \in U_i} \alpha_{i,k}
\;\le\;
e^{2S_{\max}}
\sum_{k \in U_i}
\exp\!\left(-\frac{\Gamma_h g_i}{T}(i - k)(1-l_k)\right)
\end{equation}
If, in addition, the keys under consideration satisfy $i-k\geq\Delta_{\min}$ and remain uniformly separated from full protection, $l_k\leq 1-\epsilon_l$ for some $\epsilon_l>0$, this simplifies to:
\begin{equation}
    \sum_{k \in U_i} \alpha_{i,k}
\;\le\;
|U_i|\,\exp\!\left(2S_{\max}
- \frac{\epsilon_l\Gamma_h g_i}{T}\Delta_{\min}\right)
\end{equation}
and their total attention mass vanishes exponentially as $g_i \to \infty$.
(\textit{Proof in Appendix~\ref{app:recency_bias}.})
\end{theorem}

Theorem~\ref{thm:asymptotic_eradication} highlights the joint action of the two gates: contraction depends on the landmark-adjusted distance $(i-k)(1-l_k)$ rather than on physical distance alone. Larger $l_k$ weakens the penalty, while larger $g_i$ strengthens it. Under the uniform separation condition $l_k\leq1-\epsilon_l$, this result yields an \textbf{effective context length}. For any tolerance $p\in(0,1)$, define $\Delta_i^{\text{eff},p}:=\frac{T}{\epsilon_l\Gamma_h g_i}\left(2S_{\max}+\log\frac{|U_i|}{p}\right)$.
The collective attention mass assigned to keys in $U_i$ beyond $\Delta_i^{\text{eff},p}$ is bounded by $p$. Thus, $\Delta_i^{\text{eff},p}$ defines a query- and head-dependent horizon for weakly protected context: larger $g_i$ or $\Gamma_h$ contracts this horizon, while larger $l_k$ reduces the effective distance of individual keys. Together, Theorems~\ref{thm:bilateral_invariant} and~\ref{thm:asymptotic_eradication} capture GAPE’s complementary routing effects: the landmark gate lifts selected keys toward full protection, while the query gate controls the contraction of the remaining context. Under the same separation assumptions, weakly protected distant tokens contribute only an exponentially decaying residual to the softmax normalization, limiting background dilution (Proof in Appendix~\ref{app:partition_growth}).

\noindent\textbf{Entropy Stability.}
As context length increases, additional tokens enter the softmax normalization and can spread attention mass over increasingly large portions of the sequence, weakening selective retrieval~\citep{softmaxpetar}. RoPE extrapolation can exacerbate this behavior when unseen relative phases produce unstable long-range similarities~\citep{chen2023extending,liu2023scaling}. GAPE provides a complementary control mechanism: query-dependent contraction suppresses weakly protected background tokens, while landmark activation allows selected distant keys to remain accessible. In Appendix~\ref{app:entropy_contraction}, Corollary~\ref{cor:entropy_collapse} shows that, when protected landmarks saturate and the remaining keys stay bounded away from full protection, increasing $g_i$ drives their attention mass toward zero. The distribution therefore converges to that over the protected set, whose relative weights remain governed by the semantic logits rather than by $g_i$. This limits entropy growth without uniformly suppressing distant information. We examine this behavior empirically in Section~\ref{sec:selective_routing} and Section~\ref{sec:GAPE_mechanistic_analysis}. On natural-language inputs, GAPE yields lower average attention entropy than the corresponding rotary baselines, whereas the controlled NIAH setting shows that entropy reduction is not uniform across configurations. This suggests that GAPE adapts the strength of contraction to the routing demands of each query, rather than uniformly driving attention toward lower entropy.

\begin{summarybox}
\textbf{Key Takeaways:} \textit{GAPE separates context contraction from key preservation: $g_i$ controls how strongly the current query contracts distant context, while $l_j$ continuously modulates how strongly each key is exposed to that contraction. The head-specific amplitude $\Gamma_h$ calibrates the overall strength of this structural pathway. GAPE retains selected distant information while suppressing weakly protected background context, and its factorized form is compatible with fused SDPA and FlashAttention.}
\end{summarybox}

\section{Can GAPE Mitigate Phase Hallucinations?}
\label{sec:selective_routing}
Distance-driven positional biases such as ALiBi apply a fixed penalty that grows with distance, thereby treating recency as a proxy for relevance. This can be limiting in long-context retrieval, where distant evidence may be essential, and in settings with confounders, where the required degree of recency bias may vary across queries. As discussed in Section~\ref{sec:GAPE}, GAPE avoids this coupling by separating distance-based contraction from continuous landmark protection.

In this section, we evaluate GAPE in a controlled needle-in-a-haystack (NIAH) setup and examine the behavior of its routing components under context extrapolation.

\begin{figure}[ht]
    \centering
    \includegraphics[width=0.9\linewidth]{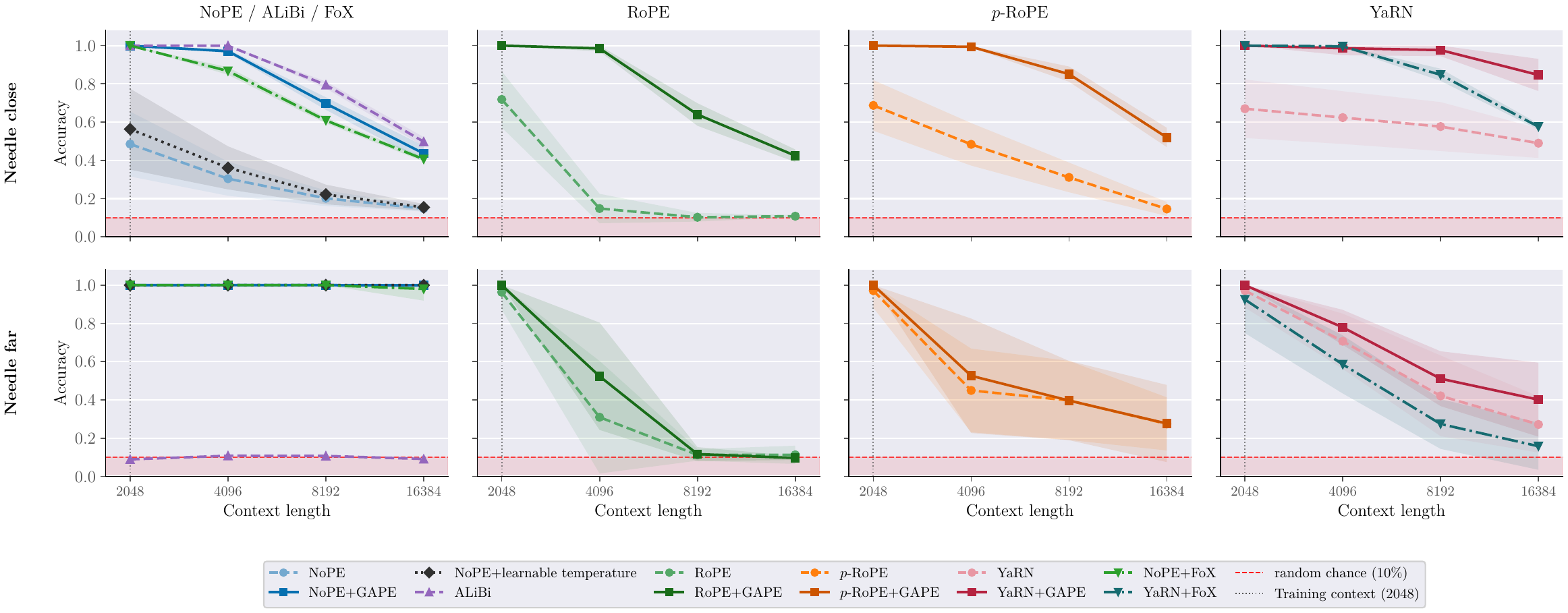}
    \caption{
    \emph{NIAH retrieval under context extrapolation.}
    Models are trained at $2048$ tokens and evaluated at $1{\times}$, $2{\times}$, $4{\times}$, and $8{\times}$ context lengths, with the target needle placed close to the query (top row) or far from it (bottom row). From left to right, we compare non-rotary and recency-based controls, RoPE, $p$-RoPE, and YaRN, including their GAPE-augmented variants where applicable.
    }
    \label{fig:GAPE_accuracy_analysis}
\end{figure}

\textbf{Phase Hallucinations} \label{sec:bounding_hallucinations}
Building on the discussion in Section~\ref{sec:related_work}, we now formalize the OOD failure mode that can arise when rotary encodings are extrapolated beyond their training context. The issue is both spectral and geometric: RoPE frequency bands are calibrated to the training horizon, so high-frequency channels are highly sensitive to small token rearrangements \cite{okafrequency, okaprobing, xu2024base}, while the lowest-frequency channels that act as semantic carriers in-distribution become fragile under extrapolation \cite{barbero2024round, wertheimer2026frayed}. As a result, relative phases at long distances can drift into configurations not encountered during training, allowing unrelated query--key pairs to align spuriously and produce artificially large semantic logits. We refer to such OOD alignments as \emph{phase hallucinations}. 

We show that GAPE provides a mechanism for suppressing these spurious long-range alignments by adjusting the effective context length, while allowing distant keys to remain accessible through landmark protection. This is formalized as follows:

\begin{theorem}
\label{thm:dynamic_and_asymptotic}
Assume that the semantic scores are uniformly bounded as $|s_{i,m}| \le S_{\max}$ and let $g_i>0$. For any key $j<i$ with $l_j<1$ and relative distance $\Delta=i-j$, the logit gap satisfies $a_{i,i}-a_{i,j}\ge -2S_{\max}+\frac{\Gamma_h g_i}{T}\Delta(1-l_j)$. Consequently, a sufficient condition for the current-position logit to exceed that of $j$ is $g_i > \frac{2S_{\max}T}{\Gamma_h\Delta(1-l_j)}$. \textit{(An extended statement and proof are provided in Appendix~\ref{app:bounding_hallucinations}).}
\end{theorem}

This result predicts that, when spurious OOD long-range alignments contribute to retrieval failures, GAPE should improve robustness under length extension by suppressing weakly protected distant keys while preserving relevant landmarks. We test this in the controlled NIAH experiments below and on natural-language inputs in Section~\ref{sec:experiments}.

\noindent\textbf{NIAH retrieval.}
We examine GAPE in a controlled NIAH setup, where a relevant token must be recovered from multiple needles embedded among distractors. Specifically, we use a vocabulary composed of 14 filler tokens, a key token `KEY', the equal sign `=', 10 digit tokens, and a query token `?'. A \textit{needle} is a 3-token triple (`KEY', =, $d$), where $d$ is a digit token. Needles are distributed throughout the sequence using the chunk-based procedure described in Appendix~\ref{sec:niah-settings}. We train 2-layer transformers at $L=2048$, with $n=\lfloor L/64\rfloor=32$ needles, to retrieve either the first or the last needle, corresponding to the far and close regimes, respectively. We then evaluate at $L\in\{2048,4096,8192,16384\}$, corresponding to $1{\times}$--$8{\times}$ the training context length, scaling $n$ accordingly.

\textbf{Results.} Figure~\ref{fig:GAPE_accuracy_analysis} shows that fixed or single-factor recency mechanisms are brittle across retrieval regimes: ALiBi performs strongly when the target is close but fails when relevant evidence is distant.
By contrast, FoX performs strongly when the target is far, but its performance degrades faster than GAPE's when the target is close. 
A learned per-head temperature also does not reproduce GAPE's behavior under extrapolation, particularly for close retrieval, indicating that adaptive routing cannot be reduced to uniform attention sharpening. Across rotary encodings, augmenting RoPE and $p$-RoPE with GAPE improves robustness beyond the training horizon. GAPE is also complementary to rotary extrapolation methods: at $16$K, YaRN+GAPE improves over YaRN from $49.0\%$ to $84.6\%$ for close retrieval and from $27.2\%$ to $40.1\%$ for far retrieval. Augmenting YaRN with FoX yields weaker performance, especially at longer contexts (8192 and 16384) and across both needle configurations. Together, these results support the complementary roles of positional extrapolation and adaptive structural routing.

\textbf{Mechanistic considerations.} Figure~\ref{fig:GAPE_gate_analysis} in Appendix~\ref{sec:niah-settings} provides a mechanistic view of this behavior. The landmark score $l_j$ develops sparse peaks at needle positions, assigning them stronger protection than surrounding filler and distractor tokens. The query gate $g_i$ also adapts to the retrieval regime: close targets induce larger values and therefore stronger contraction, whereas far targets keep $g_i$ near zero, preserving a broader effective horizon. Attention entropy also shows that GAPE does not uniformly sharpen attention but adapts its routing to the demands of the query. Additional theoretical analysis is provided in Appendix~\ref{app:niah_retrieval}.

\begin{summarybox}
\textbf{Key Takeaways:} \textit{GAPE improves retrieval robustness under OOD context extension by adapting context contraction while preserving relevant distant keys, consistent with the mitigation mechanism analyzed above.}
\end{summarybox}

\section{Experiments on Long-Context Language Modeling}
\label{sec:experiments}
\label{sec:real_world_setup}

In this section, we analyze the properties of GAPE when integrated into models trained on natural language pre-training corpora. First, we compare our combined configuration, $p$-RoPE+GAPE (referred to as GAPE throughout the small-model experiments below), against standard rotary baselines (RoPE and $p$-RoPE) under context extension. Second, we examine attention entropy and learned head-wise routing patterns to characterize how GAPE reshapes attention under long contexts. Finally, we test whether GAPE can be retrofitted into a pretrained OLMo-3 7B model and evaluate its effect on standard capabilities and long-context performance.

\textbf{OOD Perplexity.}
\label{sec:ood_ppl_context_extension}
We evaluate whether GAPE improves the robustness of language modeling when the inference context exceeds the training context length. We train a nanoGPT model (44M) for each positional encoding on a subset of Dolma 3 Longmino Mix \cite{olmo2025olmo3}. We use three dataset subsets with average document lengths of $2^{13}$, $2^{14}$, and $2^{15}$ tokens. We restrict training to the $2^{13}$ subset, truncating sequences to a 1024-token window to provide a controlled baseline and match original document sizes with the training context. The longer subsets ($2^{14}$ and $2^{15}$) are reserved for evaluation. This stratification ensures that during OOD evaluation, the models process genuinely contiguous long texts rather than artificially concatenated sequences. Reported values are calculated on \num{8000} samples. Complete training, hyperparameter, and dataset details are provided in Appendix~\ref{app:training_setup}.

We evaluate perplexity on progressively longer contexts, comparing GAPE against RoPE and $p$-RoPE. As shown in Figure~\ref{fig:ood_ppl_by_context_length}, all methods remain stable around the training context length, but their behavior diverges as the evaluation context grows. In particular, RoPE and $p$-RoPE exhibit a sharper perplexity increase at long contexts. GAPE instead yields lower perplexity in the OOD regime, especially at the largest evaluated context lengths.

\begin{figure}[h]
    \centering
    \includegraphics[width=0.8\linewidth]{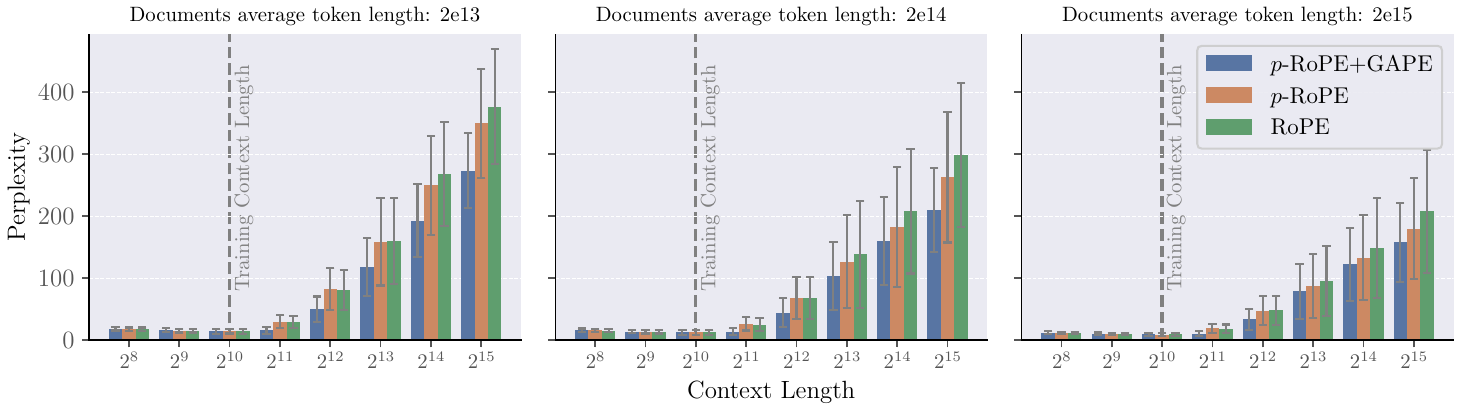}
    \caption{
    \emph{OOD perplexity under context extension} (mean $\pm$ std). We train models at a fixed context length and evaluate them on progressively longer sequences. Across all document-length regimes, perplexity increases once the evaluation context exceeds the training horizon, but GAPE grows more slowly than RoPE and $p$-RoPE, consistent with improved robustness under length extrapolation.
    }
    \label{fig:ood_ppl_by_context_length}
\end{figure}

\label{sec:attention_entropy}
For the entropy and mechanistic experiments, we use the same model trained for the OOD setup.

\noindent\textbf{Attention Maps Entropy Analysis.}
Having verified GAPE's gating dynamics in a controlled synthetic setting in Section~\ref{sec:selective_routing}, we now examine whether the same behavior emerges on real long-context inputs. Following \citet{liu2023scaling}, we use attention entropy as a proxy for attention sharpness, where lower entropy corresponds to more concentrated attention. We compute the average attention entropy per Transformer block on samples from the RULER benchmark. As shown in Figure~\ref{fig:GAPE_final_layer_analysis} (Left), GAPE produces lower-entropy attention maps than RoPE and $p$-RoPE across layers, suggesting that the learned mask suppresses diffuse background context. Moreover, Figure~\ref{fig:GAPE_final_layer_analysis} (Middle) shows an inverse relationship between the learned gate $g_i$ and attention entropy: layers with stronger gates produce sharper attention patterns. These results are consistent with the mechanism observed in the controlled analysis, where increasing $g_i$ contracts weakly protected attention mass while preserving access to relevant landmarks.

\noindent\textbf{Mechanistic Analysis of GAPE.}
\label{sec:GAPE_mechanistic_analysis}
We probe whether GAPE actively uses its routing variables by tracking the mean mask $\bar{M}_h$, landmark activation $\bar{l}_h$, query gate $\bar{g}_h$, and amplitude $\Gamma_h$ on a fixed reference batch. Full layer-wise trajectories are provided in Appendix~\ref{app:GAPE_routing_dynamics}; here we show the final-layer mask, which directly reflects the bias injected into the pre-softmax logits. As shown in Figure~\ref{fig:GAPE_final_layer_analysis} (Right), GAPE learns non-uniform head-wise masks: several heads develop strong biases, while others remain weakly biased. Since $M_{i,j}$ is added directly to the attention logits, this shows that the learned structural pathway actively reshapes attention. The full dynamics further show sparse, layer-dependent landmark activations and bounded amplitudes, suggesting that routing is primarily driven by the token-dependent gates rather than by uncontrolled growth of $\Gamma_h$. Overall, GAPE separates background suppression from landmark-mediated preservation. We further compare the distribution of frequency-channel norms in Appendix~\ref{app:qk_norm_analysis}, and ablate the components of the mask in Appendix~\ref{app:ablation}. The latter indicates that the joint gates are most beneficial in position-sensitive routing tasks, rather than uniformly across all long-context tasks.

\begin{figure}[ht]
    \centering
    \begin{subfigure}[h]{0.3\linewidth}
        \centering
        \includegraphics[width=\linewidth]{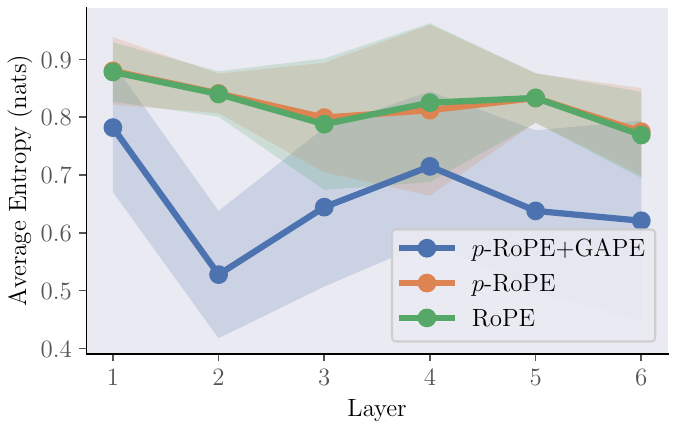}
        \label{fig:attention_entropy_by_layer_baselines}
    \end{subfigure}
    \hfill
    \begin{subfigure}[h]{0.3\linewidth}
        \centering
        \includegraphics[width=\linewidth]{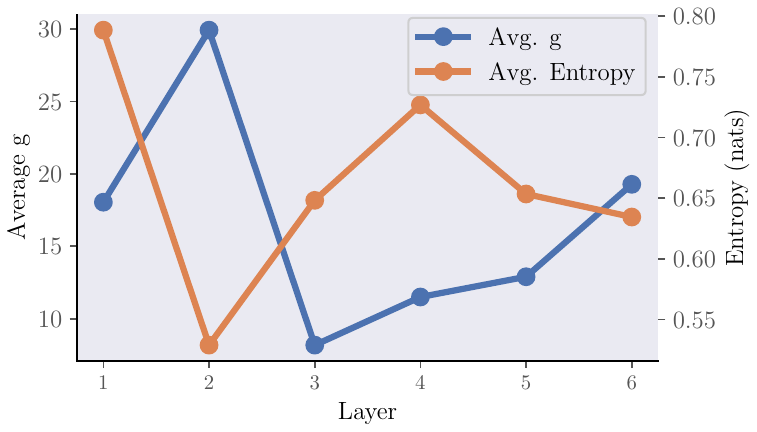}
        \label{fig:attention_entropy_by_layer}
    \end{subfigure}
    \hfill
    \begin{subfigure}[h]{0.36\linewidth}
        \centering
        \includegraphics[width=\linewidth]{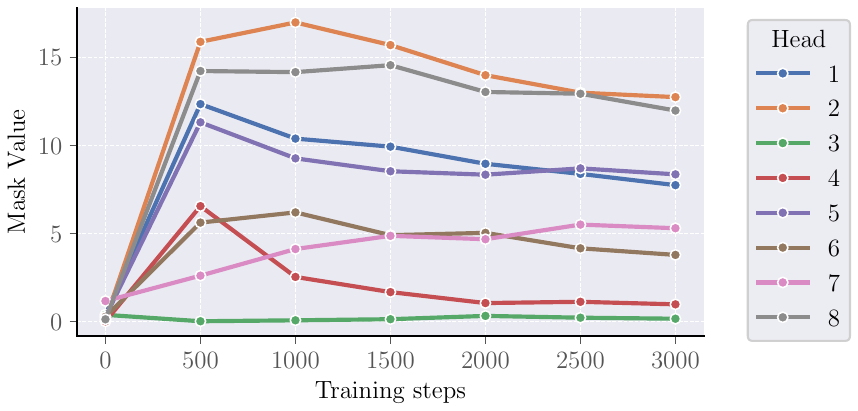}
        \label{fig:GAPE_layer5_mask}
    \end{subfigure}
    \vspace{-0.5cm}
    \caption{
    \emph{Attention sharpness and head behavior under GAPE.} \textbf{Left:} Average attention entropy across layers (mean $\pm$ std). Lower entropy indicates sharper and more concentrated attention. GAPE consistently produces lower entropy than RoPE and $p$-RoPE. \textbf{Middle:} Relationship between the learned query gate $g$ and attention entropy across
    layers. As $g$ increases, entropy decreases, showing that stronger gating is associated with more selective attention. \textbf{Right:} Head-wise mask values learned by GAPE. Different heads develop different mask strengths, indicating a specialization of routing roles, where some heads apply stronger structural suppression than others. 
    }
    \vspace{-0.5cm}
    \label{fig:GAPE_final_layer_analysis}
\end{figure}

\textbf{Retrofitting GAPE into OLMo-3.}
\label{sec:gape-olmo-retrofit}
\label{sec:long_context_benchmark}
The controlled small-model experiments isolate the positional mechanism. To test whether GAPE can be added to an existing pretrained model, we retrofit the OLMo-3 7B checkpoint immediately before its long-context adaptation (8K native context, already pretrained with RoPE). We selected OLMo-3 because its weights, intermediate checkpoints, corpus, and recipes are public and documented, making this retrofit reproducible.

We keep the pretrained backbone frozen and jointly train the GAPE parameters and rank-8 LoRA adapters in its eight full-attention layers. The LoRA adapters are applied to the query, key, value, and output projections. The LoRA-only control has approximately 2.10M trainable parameters; GAPE adds approximately 66K parameters. Both conditions use the same training recipe and process 100M Dolma 3 Longmino tokens with a maximum training context length of 12,288. This length extends beyond the native 8K window while leaving 16K and 32K for evaluation beyond the adaptation length. The structural coordinates replace one pretrained RoPE pair, keeping the head width fixed. The full adaptation uses standard FlashAttention-2. 
The adaptation configuration and evaluation details are provided in Appendix \ref{app:olmo_retrofit}.

We then evaluate in two stages. First, we use OLMES, OLMo's evaluation framework, to check whether the retrofit degrades standard capabilities relative to the LoRA-only control. The Tier-1 evaluation includes ARC, BoolQ, CSQA, HellaSwag, OBQA, PIQA, SIQA, WinoGrande, and MMLU-STEM; we report the unweighted mean over its 30 task configurations. 
Second, we evaluate long-context performance on RULER~\cite{hsieh2024ruler} at 4K, 8K, 16K, and 32K. We use 100 examples per subtask, as in \citet{nvidia2026nemotron3super}, and refer to this setting as RULER-100. We report the mean score across all 13 subtasks, covering NIAH, variable tracking, word extraction, and QA. Both models use identical evaluation examples and greedy decoding.

\begin{wrapfigure}{h}{0.55\columnwidth}
    \vspace{-0.7em}
    \centering
    \captionof{table}{OLMo-3 7B retrofit results. Scores are percentages; $\Delta$ is in percentage points. OLMES is averaged over 30 task configurations and RULER-100 over 13 subtasks with 100 examples each.}
    \label{tab:gape-olmo-retrofit}
    \label{tab:results_wide}
    
    \small
    \setlength{\tabcolsep}{3.2pt}
    \renewcommand{\arraystretch}{1.05}
    
    \resizebox{\linewidth}{!}{%
    \begin{tabular}{lrrr}
        \toprule
        \textbf{Evaluation} 
        & \textbf{LoRA-only} 
        & \textbf{GAPE+LoRA} 
        & $\mathbf{\Delta}$ \\
        \midrule
        OLMES Tier-1 mean & $66.002$ & $66.056$ & $+0.054$ \\
        RULER-100, 4K     & $93.88$  & $93.71$  & $-0.17$  \\
        RULER-100, 8K     & $80.52$  & $79.71$  & $-0.81$  \\
        RULER-100, 16K    & $59.02$  & $59.24$  & $+0.22$  \\
        RULER-100, 32K    & $11.62$  & $13.17$  & $+1.55$  \\
        \bottomrule
    \end{tabular}%
    }
    \vspace{-0.8em}
\end{wrapfigure}

As shown in Table~\ref{tab:gape-olmo-retrofit}, OLMES Tier-1 is essentially unchanged, giving no evidence of broad capability loss relative to the LoRA-only control on the evaluated tasks. On RULER, the models are closely matched through 16K context length, with differences within 0.81 percentage points. At 32K, GAPE improves 8 of 13 subtasks and all four task-group means, with an aggregate gain of 1.55 percentage points, although absolute performance at this length remains low. These results provide evidence that GAPE can be introduced after pretraining at a larger model scale, with comparable performance on the evaluated standard capabilities and selected gains under long-context extrapolation. The comparison uses a single 100M-token adaptation run per condition with a frozen, RoPE-pretrained backbone.
\vspace{-0.2cm}
\section{Conclusions and Limitations}
\vspace{-0.2cm}
\label{sec:conclusion}
We introduced GAPE, a gated positional encoding that augments RoPE with a content-aware mask. Instead of modifying rotary geometry, GAPE separates local similarity from long-range context control through a query-dependent gate that contracts weakly protected context and a key-dependent gate that continuously protects selected distant tokens. Our analysis shows that this decouples distance from importance, while the experiments demonstrate improved OOD length generalization across controlled retrieval and language modeling, including extrapolation up to $8\times$ the training context and compatibility with rotary interventions such as $p$-RoPE and YaRN. We further show that GAPE can be retrofitted into a pretrained 7B model without degrading the evaluated standard capabilities, while providing selected gains at longer contexts. Our large-scale evaluation is limited to a single pretrained 7B checkpoint and lightweight 100M-token adaptation, with absolute performance remaining low at 32K. Whether these gains persist across broader model families and tasks such as multi-step reasoning or long-document summarization remains an open question. Future work can also further investigate the interaction between adaptive structural routing and alternative long-context positional methods.

\bibliographystyle{iclr2027_conference}
\bibliography{iclr2027_conference}

\begin{thebibliography}{30}
\providecommand{\natexlab}[1]{#1}
\providecommand{\url}[1]{\texttt{#1}}
\expandafter\ifx\csname urlstyle\endcsname\relax
  \providecommand{\doi}[1]{doi: #1}\else
  \providecommand{\doi}{doi: \begingroup \urlstyle{rm}\Url}\fi

\bibitem[Barbero et~al.(2024)Barbero, Vitvitskyi, Perivolaropoulos, Pascanu, and Veli{\v{c}}kovi{\'c}]{barbero2024round}
Federico Barbero, Alex Vitvitskyi, Christos Perivolaropoulos, Razvan Pascanu, and Petar Veli{\v{c}}kovi{\'c}.
\newblock Round and round we go! what makes rotary positional encodings useful?
\newblock \emph{arXiv preprint arXiv:2410.06205}, 2024.

\bibitem[Chen et~al.(2023)Chen, Wong, Chen, and Tian]{chen2023extending}
Shouyuan Chen, Sherman Wong, Liangjian Chen, and Yuandong Tian.
\newblock Extending context window of large language models via positional interpolation.
\newblock \emph{arXiv preprint arXiv:2306.15595}, 2023.

\bibitem[Chi et~al.(2022)Chi, Fan, Ramadge, and Rudnicky]{chi2022kerple}
Ta-Chung Chi, Ting-Han Fan, Peter~J Ramadge, and Alexander Rudnicky.
\newblock Kerple: Kernelized relative positional embedding for length extrapolation.
\newblock \emph{Advances in Neural Information Processing Systems}, 35:\penalty0 8386--8399, 2022.

\bibitem[Dai et~al.(2019)Dai, Yang, Yang, Carbonell, Le, and Salakhutdinov]{dai2019transformer}
Zihang Dai, Zhilin Yang, Yiming Yang, Jaime~G Carbonell, Quoc Le, and Ruslan Salakhutdinov.
\newblock Transformer-xl: Attentive language models beyond a fixed-length context.
\newblock In \emph{Proceedings of the 57th annual meeting of the association for computational linguistics}, pp.\  2978--2988, 2019.

\bibitem[Dao(2023)]{dao2023flashattention2fasterattentionbetter}
Tri Dao.
\newblock Flashattention-2: Faster attention with better parallelism and work partitioning, 2023.
\newblock URL \url{https://arxiv.org/abs/2307.08691}.

\bibitem[Dao et~al.(2022)Dao, Fu, Ermon, Rudra, and Ré]{dao2022flashattentionfastmemoryefficientexact}
Tri Dao, Daniel~Y. Fu, Stefano Ermon, Atri Rudra, and Christopher Ré.
\newblock Flashattention: Fast and memory-efficient exact attention with io-awareness, 2022.
\newblock URL \url{https://arxiv.org/abs/2205.14135}.

\bibitem[Devlin et~al.(2019)Devlin, Chang, Lee, and Toutanova]{devlin2019bert}
Jacob Devlin, Ming-Wei Chang, Kenton Lee, and Kristina Toutanova.
\newblock Bert: Pre-training of deep bidirectional transformers for language understanding.
\newblock In \emph{Proceedings of the 2019 conference of the North American chapter of the association for computational linguistics: human language technologies, volume 1 (long and short papers)}, pp.\  4171--4186, 2019.

\bibitem[Ding et~al.(2024)Ding, Zhang, Zhang, Xu, Shang, Xu, Yang, and Yang]{ding2024longrope}
Yiran Ding, Li~Lyna Zhang, Chengruidong Zhang, Yuanyuan Xu, Ning Shang, Jiahang Xu, Fan Yang, and Mao Yang.
\newblock Longrope: Extending llm context window beyond 2 million tokens.
\newblock \emph{arXiv preprint arXiv:2402.13753}, 2024.

\bibitem[Grattafiori et~al.(2024)Grattafiori, Dubey, Jauhri, Pandey, Kadian, Al-Dahle, Letman, Mathur, Schelten, Vaughan, Yang, Fan, Goyal, Hartshorn, Yang, Mitra, Sravankumar, Korenev, Hinsvark, Rao, Zhang, Rodriguez, Gregerson, Spataru, Roziere, Biron, Tang, Chern, Caucheteux, Nayak, Bi, Marra, McConnell, Keller, Touret, Wu, Wong, Ferrer, Nikolaidis, Allonsius, Song, Pintz, Livshits, Wyatt, Esiobu, Choudhary, Mahajan, Garcia-Olano, Perino, Hupkes, Lakomkin, AlBadawy, Lobanova, Dinan, Smith, Radenovic, Guzmán, Zhang, Synnaeve, Lee, Anderson, Thattai, Nail, Mialon, Pang, Cucurell, Nguyen, Korevaar, Xu, Touvron, Zarov, Ibarra, Kloumann, Misra, Evtimov, Zhang, Copet, Lee, Geffert, Vranes, Park, Mahadeokar, Shah, van~der Linde, Billock, Hong, Lee, Fu, Chi, Huang, Liu, Wang, Yu, Bitton, Spisak, Park, Rocca, Johnstun, Saxe, Jia, Alwala, Prasad, Upasani, Plawiak, Li, Heafield, Stone, El-Arini, Iyer, Malik, Chiu, Bhalla, Lakhotia, Rantala-Yeary, van~der Maaten, Chen, Tan, Jenkins, Martin, Madaan, Malo, Blecher,
  Landzaat, de~Oliveira, Muzzi, Pasupuleti, Singh, Paluri, Kardas, Tsimpoukelli, Oldham, Rita, Pavlova, Kambadur, Lewis, Si, Singh, Hassan, Goyal, Torabi, Bashlykov, Bogoychev, Chatterji, Zhang, Duchenne, Çelebi, Alrassy, Zhang, Li, Vasic, Weng, Bhargava, Dubal, Krishnan, Koura, Xu, He, Dong, Srinivasan, Ganapathy, Calderer, Cabral, Stojnic, Raileanu, Maheswari, Girdhar, Patel, Sauvestre, Polidoro, Sumbaly, Taylor, Silva, Hou, Wang, Hosseini, Chennabasappa, Singh, Bell, Kim, Edunov, Nie, Narang, Raparthy, Shen, Wan, Bhosale, Zhang, Vandenhende, Batra, Whitman, Sootla, Collot, Gururangan, Borodinsky, Herman, Fowler, Sheasha, Georgiou, Scialom, Speckbacher, Mihaylov, Xiao, Karn, Goswami, Gupta, Ramanathan, Kerkez, Gonguet, Do, Vogeti, Albiero, Petrovic, Chu, Xiong, Fu, Meers, Martinet, Wang, Wang, Tan, Xia, Xie, Jia, Wang, Goldschlag, Gaur, Babaei, Wen, Song, Zhang, Li, Mao, Coudert, Yan, Chen, Papakipos, Singh, Srivastava, Jain, Kelsey, Shajnfeld, Gangidi, Victoria, Goldstand, Menon, Sharma, Boesenberg,
  Baevski, Feinstein, Kallet, Sangani, Teo, Yunus, Lupu, Alvarado, Caples, Gu, Ho, Poulton, Ryan, Ramchandani, Dong, Franco, Goyal, Saraf, Chowdhury, Gabriel, Bharambe, Eisenman, Yazdan, James, Maurer, Leonhardi, Huang, Loyd, Paola, Paranjape, Liu, Wu, Ni, Hancock, Wasti, Spence, Stojkovic, Gamido, Montalvo, Parker, Burton, Mejia, Liu, Wang, Kim, Zhou, Hu, Chu, Cai, Tindal, Feichtenhofer, Gao, Civin, Beaty, Kreymer, Li, Adkins, Xu, Testuggine, David, Parikh, Liskovich, Foss, Wang, Le, Holland, Dowling, Jamil, Montgomery, Presani, Hahn, Wood, Le, Brinkman, Arcaute, Dunbar, Smothers, Sun, Kreuk, Tian, Kokkinos, Ozgenel, Caggioni, Kanayet, Seide, Florez, Schwarz, Badeer, Swee, Halpern, Herman, Sizov, Guangyi, Zhang, Lakshminarayanan, Inan, Shojanazeri, Zou, Wang, Zha, Habeeb, Rudolph, Suk, Aspegren, Goldman, Zhan, Damlaj, Molybog, Tufanov, Leontiadis, Veliche, Gat, Weissman, Geboski, Kohli, Lam, Asher, Gaya, Marcus, Tang, Chan, Zhen, Reizenstein, Teboul, Zhong, Jin, Yang, Cummings, Carvill, Shepard, McPhie,
  Torres, Ginsburg, Wang, Wu, U, Saxena, Khandelwal, Zand, Matosich, Veeraraghavan, Michelena, Li, Jagadeesh, Huang, Chawla, Huang, Chen, Garg, A, Silva, Bell, Zhang, Guo, Yu, Moshkovich, Wehrstedt, Khabsa, Avalani, Bhatt, Mankus, Hasson, Lennie, Reso, Groshev, Naumov, Lathi, Keneally, Liu, Seltzer, Valko, Restrepo, Patel, Vyatskov, Samvelyan, Clark, Macey, Wang, Hermoso, Metanat, Rastegari, Bansal, Santhanam, Parks, White, Bawa, Singhal, Egebo, Usunier, Mehta, Laptev, Dong, Cheng, Chernoguz, Hart, Salpekar, Kalinli, Kent, Parekh, Saab, Balaji, Rittner, Bontrager, Roux, Dollar, Zvyagina, Ratanchandani, Yuvraj, Liang, Alao, Rodriguez, Ayub, Murthy, Nayani, Mitra, Parthasarathy, Li, Hogan, Battey, Wang, Howes, Rinott, Mehta, Siby, Bondu, Datta, Chugh, Hunt, Dhillon, Sidorov, Pan, Mahajan, Verma, Yamamoto, Ramaswamy, Lindsay, Lindsay, Feng, Lin, Zha, Patil, Shankar, Zhang, Zhang, Wang, Agarwal, Sajuyigbe, Chintala, Max, Chen, Kehoe, Satterfield, Govindaprasad, Gupta, Deng, Cho, Virk, Subramanian, Choudhury,
  Goldman, Remez, Glaser, Best, Koehler, Robinson, Li, Zhang, Matthews, Chou, Shaked, Vontimitta, Ajayi, Montanez, Mohan, Kumar, Mangla, Ionescu, Poenaru, Mihailescu, Ivanov, Li, Wang, Jiang, Bouaziz, Constable, Tang, Wu, Wang, Wu, Gao, Kleinman, Chen, Hu, Jia, Qi, Li, Zhang, Zhang, Adi, Nam, Yu, Wang, Zhao, Hao, Qian, Li, He, Rait, DeVito, Rosnbrick, Wen, Yang, Zhao, and Ma]{grattafiori2024llama3herdmodels}
Aaron Grattafiori, Abhimanyu Dubey, Abhinav Jauhri, Abhinav Pandey, Abhishek Kadian, Ahmad Al-Dahle, Aiesha Letman, Akhil Mathur, Alan Schelten, Alex Vaughan, Amy Yang, Angela Fan, Anirudh Goyal, Anthony Hartshorn, Aobo Yang, Archi Mitra, Archie Sravankumar, Artem Korenev, Arthur Hinsvark, Arun Rao, Aston Zhang, Aurelien Rodriguez, Austen Gregerson, Ava Spataru, Baptiste Roziere, Bethany Biron, Binh Tang, Bobbie Chern, Charlotte Caucheteux, Chaya Nayak, Chloe Bi, Chris Marra, Chris McConnell, Christian Keller, Christophe Touret, Chunyang Wu, Corinne Wong, Cristian~Canton Ferrer, Cyrus Nikolaidis, Damien Allonsius, Daniel Song, Danielle Pintz, Danny Livshits, Danny Wyatt, David Esiobu, Dhruv Choudhary, Dhruv Mahajan, Diego Garcia-Olano, Diego Perino, Dieuwke Hupkes, Egor Lakomkin, Ehab AlBadawy, Elina Lobanova, Emily Dinan, Eric~Michael Smith, Filip Radenovic, Francisco Guzmán, Frank Zhang, Gabriel Synnaeve, Gabrielle Lee, Georgia~Lewis Anderson, Govind Thattai, Graeme Nail, Gregoire Mialon, Guan Pang,
  Guillem Cucurell, Hailey Nguyen, Hannah Korevaar, Hu~Xu, Hugo Touvron, Iliyan Zarov, Imanol~Arrieta Ibarra, Isabel Kloumann, Ishan Misra, Ivan Evtimov, Jack Zhang, Jade Copet, Jaewon Lee, Jan Geffert, Jana Vranes, Jason Park, Jay Mahadeokar, Jeet Shah, Jelmer van~der Linde, Jennifer Billock, Jenny Hong, Jenya Lee, Jeremy Fu, Jianfeng Chi, Jianyu Huang, Jiawen Liu, Jie Wang, Jiecao Yu, Joanna Bitton, Joe Spisak, Jongsoo Park, Joseph Rocca, Joshua Johnstun, Joshua Saxe, Junteng Jia, Kalyan~Vasuden Alwala, Karthik Prasad, Kartikeya Upasani, Kate Plawiak, Ke~Li, Kenneth Heafield, Kevin Stone, Khalid El-Arini, Krithika Iyer, Kshitiz Malik, Kuenley Chiu, Kunal Bhalla, Kushal Lakhotia, Lauren Rantala-Yeary, Laurens van~der Maaten, Lawrence Chen, Liang Tan, Liz Jenkins, Louis Martin, Lovish Madaan, Lubo Malo, Lukas Blecher, Lukas Landzaat, Luke de~Oliveira, Madeline Muzzi, Mahesh Pasupuleti, Mannat Singh, Manohar Paluri, Marcin Kardas, Maria Tsimpoukelli, Mathew Oldham, Mathieu Rita, Maya Pavlova, Melanie Kambadur,
  Mike Lewis, Min Si, Mitesh~Kumar Singh, Mona Hassan, Naman Goyal, Narjes Torabi, Nikolay Bashlykov, Nikolay Bogoychev, Niladri Chatterji, Ning Zhang, Olivier Duchenne, Onur Çelebi, Patrick Alrassy, Pengchuan Zhang, Pengwei Li, Petar Vasic, Peter Weng, Prajjwal Bhargava, Pratik Dubal, Praveen Krishnan, Punit~Singh Koura, Puxin Xu, Qing He, Qingxiao Dong, Ragavan Srinivasan, Raj Ganapathy, Ramon Calderer, Ricardo~Silveira Cabral, Robert Stojnic, Roberta Raileanu, Rohan Maheswari, Rohit Girdhar, Rohit Patel, Romain Sauvestre, Ronnie Polidoro, Roshan Sumbaly, Ross Taylor, Ruan Silva, Rui Hou, Rui Wang, Saghar Hosseini, Sahana Chennabasappa, Sanjay Singh, Sean Bell, Seohyun~Sonia Kim, Sergey Edunov, Shaoliang Nie, Sharan Narang, Sharath Raparthy, Sheng Shen, Shengye Wan, Shruti Bhosale, Shun Zhang, Simon Vandenhende, Soumya Batra, Spencer Whitman, Sten Sootla, Stephane Collot, Suchin Gururangan, Sydney Borodinsky, Tamar Herman, Tara Fowler, Tarek Sheasha, Thomas Georgiou, Thomas Scialom, Tobias Speckbacher,
  Todor Mihaylov, Tong Xiao, Ujjwal Karn, Vedanuj Goswami, Vibhor Gupta, Vignesh Ramanathan, Viktor Kerkez, Vincent Gonguet, Virginie Do, Vish Vogeti, Vítor Albiero, Vladan Petrovic, Weiwei Chu, Wenhan Xiong, Wenyin Fu, Whitney Meers, Xavier Martinet, Xiaodong Wang, Xiaofang Wang, Xiaoqing~Ellen Tan, Xide Xia, Xinfeng Xie, Xuchao Jia, Xuewei Wang, Yaelle Goldschlag, Yashesh Gaur, Yasmine Babaei, Yi~Wen, Yiwen Song, Yuchen Zhang, Yue Li, Yuning Mao, Zacharie~Delpierre Coudert, Zheng Yan, Zhengxing Chen, Zoe Papakipos, Aaditya Singh, Aayushi Srivastava, Abha Jain, Adam Kelsey, Adam Shajnfeld, Adithya Gangidi, Adolfo Victoria, Ahuva Goldstand, Ajay Menon, Ajay Sharma, Alex Boesenberg, Alexei Baevski, Allie Feinstein, Amanda Kallet, Amit Sangani, Amos Teo, Anam Yunus, Andrei Lupu, Andres Alvarado, Andrew Caples, Andrew Gu, Andrew Ho, Andrew Poulton, Andrew Ryan, Ankit Ramchandani, Annie Dong, Annie Franco, Anuj Goyal, Aparajita Saraf, Arkabandhu Chowdhury, Ashley Gabriel, Ashwin Bharambe, Assaf Eisenman, Azadeh
  Yazdan, Beau James, Ben Maurer, Benjamin Leonhardi, Bernie Huang, Beth Loyd, Beto~De Paola, Bhargavi Paranjape, Bing Liu, Bo~Wu, Boyu Ni, Braden Hancock, Bram Wasti, Brandon Spence, Brani Stojkovic, Brian Gamido, Britt Montalvo, Carl Parker, Carly Burton, Catalina Mejia, Ce~Liu, Changhan Wang, Changkyu Kim, Chao Zhou, Chester Hu, Ching-Hsiang Chu, Chris Cai, Chris Tindal, Christoph Feichtenhofer, Cynthia Gao, Damon Civin, Dana Beaty, Daniel Kreymer, Daniel Li, David Adkins, David Xu, Davide Testuggine, Delia David, Devi Parikh, Diana Liskovich, Didem Foss, Dingkang Wang, Duc Le, Dustin Holland, Edward Dowling, Eissa Jamil, Elaine Montgomery, Eleonora Presani, Emily Hahn, Emily Wood, Eric-Tuan Le, Erik Brinkman, Esteban Arcaute, Evan Dunbar, Evan Smothers, Fei Sun, Felix Kreuk, Feng Tian, Filippos Kokkinos, Firat Ozgenel, Francesco Caggioni, Frank Kanayet, Frank Seide, Gabriela~Medina Florez, Gabriella Schwarz, Gada Badeer, Georgia Swee, Gil Halpern, Grant Herman, Grigory Sizov, Guangyi, Zhang, Guna
  Lakshminarayanan, Hakan Inan, Hamid Shojanazeri, Han Zou, Hannah Wang, Hanwen Zha, Haroun Habeeb, Harrison Rudolph, Helen Suk, Henry Aspegren, Hunter Goldman, Hongyuan Zhan, Ibrahim Damlaj, Igor Molybog, Igor Tufanov, Ilias Leontiadis, Irina-Elena Veliche, Itai Gat, Jake Weissman, James Geboski, James Kohli, Janice Lam, Japhet Asher, Jean-Baptiste Gaya, Jeff Marcus, Jeff Tang, Jennifer Chan, Jenny Zhen, Jeremy Reizenstein, Jeremy Teboul, Jessica Zhong, Jian Jin, Jingyi Yang, Joe Cummings, Jon Carvill, Jon Shepard, Jonathan McPhie, Jonathan Torres, Josh Ginsburg, Junjie Wang, Kai Wu, Kam~Hou U, Karan Saxena, Kartikay Khandelwal, Katayoun Zand, Kathy Matosich, Kaushik Veeraraghavan, Kelly Michelena, Keqian Li, Kiran Jagadeesh, Kun Huang, Kunal Chawla, Kyle Huang, Lailin Chen, Lakshya Garg, Lavender A, Leandro Silva, Lee Bell, Lei Zhang, Liangpeng Guo, Licheng Yu, Liron Moshkovich, Luca Wehrstedt, Madian Khabsa, Manav Avalani, Manish Bhatt, Martynas Mankus, Matan Hasson, Matthew Lennie, Matthias Reso, Maxim
  Groshev, Maxim Naumov, Maya Lathi, Meghan Keneally, Miao Liu, Michael~L. Seltzer, Michal Valko, Michelle Restrepo, Mihir Patel, Mik Vyatskov, Mikayel Samvelyan, Mike Clark, Mike Macey, Mike Wang, Miquel~Jubert Hermoso, Mo~Metanat, Mohammad Rastegari, Munish Bansal, Nandhini Santhanam, Natascha Parks, Natasha White, Navyata Bawa, Nayan Singhal, Nick Egebo, Nicolas Usunier, Nikhil Mehta, Nikolay~Pavlovich Laptev, Ning Dong, Norman Cheng, Oleg Chernoguz, Olivia Hart, Omkar Salpekar, Ozlem Kalinli, Parkin Kent, Parth Parekh, Paul Saab, Pavan Balaji, Pedro Rittner, Philip Bontrager, Pierre Roux, Piotr Dollar, Polina Zvyagina, Prashant Ratanchandani, Pritish Yuvraj, Qian Liang, Rachad Alao, Rachel Rodriguez, Rafi Ayub, Raghotham Murthy, Raghu Nayani, Rahul Mitra, Rangaprabhu Parthasarathy, Raymond Li, Rebekkah Hogan, Robin Battey, Rocky Wang, Russ Howes, Ruty Rinott, Sachin Mehta, Sachin Siby, Sai~Jayesh Bondu, Samyak Datta, Sara Chugh, Sara Hunt, Sargun Dhillon, Sasha Sidorov, Satadru Pan, Saurabh Mahajan,
  Saurabh Verma, Seiji Yamamoto, Sharadh Ramaswamy, Shaun Lindsay, Shaun Lindsay, Sheng Feng, Shenghao Lin, Shengxin~Cindy Zha, Shishir Patil, Shiva Shankar, Shuqiang Zhang, Shuqiang Zhang, Sinong Wang, Sneha Agarwal, Soji Sajuyigbe, Soumith Chintala, Stephanie Max, Stephen Chen, Steve Kehoe, Steve Satterfield, Sudarshan Govindaprasad, Sumit Gupta, Summer Deng, Sungmin Cho, Sunny Virk, Suraj Subramanian, Sy~Choudhury, Sydney Goldman, Tal Remez, Tamar Glaser, Tamara Best, Thilo Koehler, Thomas Robinson, Tianhe Li, Tianjun Zhang, Tim Matthews, Timothy Chou, Tzook Shaked, Varun Vontimitta, Victoria Ajayi, Victoria Montanez, Vijai Mohan, Vinay~Satish Kumar, Vishal Mangla, Vlad Ionescu, Vlad Poenaru, Vlad~Tiberiu Mihailescu, Vladimir Ivanov, Wei Li, Wenchen Wang, Wenwen Jiang, Wes Bouaziz, Will Constable, Xiaocheng Tang, Xiaojian Wu, Xiaolan Wang, Xilun Wu, Xinbo Gao, Yaniv Kleinman, Yanjun Chen, Ye~Hu, Ye~Jia, Ye~Qi, Yenda Li, Yilin Zhang, Ying Zhang, Yossi Adi, Youngjin Nam, Yu, Wang, Yu~Zhao, Yuchen Hao, Yundi
  Qian, Yunlu Li, Yuzi He, Zach Rait, Zachary DeVito, Zef Rosnbrick, Zhaoduo Wen, Zhenyu Yang, Zhiwei Zhao, and Zhiyu Ma.
\newblock The llama 3 herd of models, 2024.
\newblock URL \url{https://arxiv.org/abs/2407.21783}.

\bibitem[Hsieh et~al.(2024)Hsieh, Sun, Kriman, Acharya, Rekesh, Jia, Zhang, and Ginsburg]{hsieh2024ruler}
Cheng-Ping Hsieh, Simeng Sun, Samuel Kriman, Shantanu Acharya, Dima Rekesh, Fei Jia, Yang Zhang, and Boris Ginsburg.
\newblock Ruler: What's the real context size of your long-context language models?
\newblock \emph{arXiv preprint arXiv:2404.06654}, 2024.

\bibitem[Li et~al.(2023)Li, You, Guruganesh, Ainslie, Ontanon, Zaheer, Sanghai, Yang, Kumar, and Bhojanapalli]{li2023functional}
Shanda Li, Chong You, Guru Guruganesh, Joshua Ainslie, Santiago Ontanon, Manzil Zaheer, Sumit Sanghai, Yiming Yang, Sanjiv Kumar, and Srinadh Bhojanapalli.
\newblock Functional interpolation for relative positions improves long context transformers.
\newblock \emph{arXiv preprint arXiv:2310.04418}, 2023.

\bibitem[Lin et~al.(2025)Lin, Nikishin, He, and Courville]{lin2025forgettingtransformersoftmaxattention}
Zhixuan Lin, Evgenii Nikishin, Xu~Owen He, and Aaron Courville.
\newblock Forgetting transformer: Softmax attention with a forget gate, 2025.
\newblock URL \url{https://arxiv.org/abs/2503.02130}.

\bibitem[Liu et~al.(2023)Liu, Yan, Zhang, An, Qiu, and Lin]{liu2023scaling}
Xiaoran Liu, Hang Yan, Shuo Zhang, Chenxin An, Xipeng Qiu, and Dahua Lin.
\newblock Scaling laws of rope-based extrapolation.
\newblock \emph{arXiv preprint arXiv:2310.05209}, 2023.

\bibitem[Men et~al.(2024)Men, Xu, Wang, Zhang, Lin, Han, and Chen]{men2024base}
Xin Men, Mingyu Xu, Bingning Wang, Qingyu Zhang, Hongyu Lin, Xianpei Han, and Weipeng Chen.
\newblock Base of rope bounds context length.
\newblock \emph{arXiv preprint arXiv:2405.14591}, 2024.

\bibitem[{NVIDIA}(2026)]{nvidia2026nemotron3super}
{NVIDIA}.
\newblock {Nemotron 3 Super}: Open, efficient mixture-of-experts hybrid mamba-transformer model for agentic reasoning.
\newblock Technical report, NVIDIA, 2026.
\newblock URL \url{https://research.nvidia.com/labs/nemotron/files/NVIDIA-Nemotron-3-Super-Technical-Report.pdf}.

\bibitem[Oka et~al.({\natexlab{a}})Oka, Hanafusa, Hasegawa, Nishida, and Saito]{okaprobing}
Yui Oka, Kentaro Hanafusa, Taku Hasegawa, Kyosuke Nishida, and Kuniko Saito.
\newblock Probing rotary position embeddings through frequency entropy.
\newblock In \emph{The Fourteenth International Conference on Learning Representations}, {\natexlab{a}}.

\bibitem[Oka et~al.({\natexlab{b}})Oka, Saito, Nishida, and Saito]{okafrequency}
Yui Oka, Itsumi Saito, Kyosuke Nishida, and Kuniko Saito.
\newblock Frequency bands in rope: Base frequency and context length shape the interpolation--extrapolation trade-off.
\newblock In \emph{The Fourteenth International Conference on Learning Representations}, {\natexlab{b}}.

\bibitem[Olmo et~al.(2025)Olmo, Ettinger, Bertsch, Kuehl, Graham, Heineman, Groeneveld, Brahman, Timbers, Ivison, Morrison, Poznanski, Lo, Soldaini, Jordan, Chen, Noukhovitch, Lambert, Walsh, Dasigi, Berry, Malik, Shah, Geng, Arora, Gupta, Anderson, Xiao, Murray, Romero, Graf, Asai, Bhagia, Wettig, Liu, Rangapur, Anastasiades, Huang, Schwenk, Trivedi, Magnusson, Lochner, Liu, Miranda, Sap, Morgan, Schmitz, Guerquin, Wilson, Huff, Bras, Xin, Shao, Skjonsberg, Shen, Li, Wilde, Pyatkin, Merrill, Chang, Gu, Zeng, Sabharwal, Zettlemoyer, Koh, Farhadi, Smith, and Hajishirzi]{olmo2025olmo3}
Team Olmo, Allyson Ettinger, Amanda Bertsch, Bailey Kuehl, David Graham, David Heineman, Dirk Groeneveld, Faeze Brahman, Finbarr Timbers, Hamish Ivison, Jacob Morrison, Jake Poznanski, Kyle Lo, Luca Soldaini, Matt Jordan, Mayee Chen, Michael Noukhovitch, Nathan Lambert, Pete Walsh, Pradeep Dasigi, Robert Berry, Saumya Malik, Saurabh Shah, Scott Geng, Shane Arora, Shashank Gupta, Taira Anderson, Teng Xiao, Tyler Murray, Tyler Romero, Victoria Graf, Akari Asai, Akshita Bhagia, Alexander Wettig, Alisa Liu, Aman Rangapur, Chloe Anastasiades, Costa Huang, Dustin Schwenk, Harsh Trivedi, Ian Magnusson, Jaron Lochner, Jiacheng Liu, Lester James~V. Miranda, Maarten Sap, Malia Morgan, Michael Schmitz, Michal Guerquin, Michael Wilson, Regan Huff, Ronan~Le Bras, Rui Xin, Rulin Shao, Sam Skjonsberg, Shannon~Zejiang Shen, Shuyue~Stella Li, Tucker Wilde, Valentina Pyatkin, Will Merrill, Yapei Chang, Yuling Gu, Zhiyuan Zeng, Ashish Sabharwal, Luke Zettlemoyer, Pang~Wei Koh, Ali Farhadi, Noah~A. Smith, and Hannaneh
  Hajishirzi.
\newblock Olmo 3, 2025.
\newblock URL \url{https://arxiv.org/abs/2512.13961}.

\bibitem[Peng et~al.(2023)Peng, Quesnelle, Fan, and Shippole]{peng2023yarn}
Bowen Peng, Jeffrey Quesnelle, Honglu Fan, and Enrico Shippole.
\newblock Yarn: Efficient context window extension of large language models.
\newblock \emph{arXiv preprint arXiv:2309.00071}, 2023.

\bibitem[Press et~al.(2021)Press, Smith, and Lewis]{press2021train}
Ofir Press, Noah~A Smith, and Mike Lewis.
\newblock Train short, test long: Attention with linear biases enables input length extrapolation.
\newblock \emph{arXiv preprint arXiv:2108.12409}, 2021.

\bibitem[Raffel et~al.(2020)Raffel, Shazeer, Roberts, Lee, Narang, Matena, Zhou, Li, and Liu]{raffel2020exploring}
Colin Raffel, Noam Shazeer, Adam Roberts, Katherine Lee, Sharan Narang, Michael Matena, Yanqi Zhou, Wei Li, and Peter~J Liu.
\newblock Exploring the limits of transfer learning with a unified text-to-text transformer.
\newblock \emph{Journal of machine learning research}, 21\penalty0 (140):\penalty0 1--67, 2020.

\bibitem[Shah et~al.(2024)Shah, Bikshandi, Zhang, Thakkar, Ramani, and Dao]{shah2024flashattention}
Jay Shah, Ganesh Bikshandi, Ying Zhang, Vijay Thakkar, Pradeep Ramani, and Tri Dao.
\newblock Flashattention-3: Fast and accurate attention with asynchrony and low-precision.
\newblock \emph{Advances in Neural Information Processing Systems}, 37:\penalty0 68658--68685, 2024.

\bibitem[Shaw et~al.(2018)Shaw, Uszkoreit, and Vaswani]{shaw2018self}
Peter Shaw, Jakob Uszkoreit, and Ashish Vaswani.
\newblock Self-attention with relative position representations.
\newblock In \emph{Proceedings of the 2018 Conference of the North American Chapter of the Association for Computational Linguistics: Human Language Technologies, Volume 2 (Short Papers)}, pp.\  464--468, 2018.

\bibitem[Su et~al.(2024)Su, Ahmed, Lu, Pan, Bo, and Liu]{su2024roformer}
Jianlin Su, Murtadha Ahmed, Yu~Lu, Shengfeng Pan, Wen Bo, and Yunfeng Liu.
\newblock Roformer: Enhanced transformer with rotary position embedding.
\newblock \emph{Neurocomputing}, 568:\penalty0 127063, 2024.

\bibitem[Team et~al.(2024)Team, Mesnard, Hardin, Dadashi, Bhupatiraju, Pathak, Sifre, Rivière, Kale, Love, Tafti, Hussenot, Sessa, Chowdhery, Roberts, Barua, Botev, Castro-Ros, Slone, Héliou, Tacchetti, Bulanova, Paterson, Tsai, Shahriari, Lan, Choquette-Choo, Crepy, Cer, Ippolito, Reid, Buchatskaya, Ni, Noland, Yan, Tucker, Muraru, Rozhdestvenskiy, Michalewski, Tenney, Grishchenko, Austin, Keeling, Labanowski, Lespiau, Stanway, Brennan, Chen, Ferret, Chiu, Mao-Jones, Lee, Yu, Millican, Sjoesund, Lee, Dixon, Reid, Mikuła, Wirth, Sharman, Chinaev, Thain, Bachem, Chang, Wahltinez, Bailey, Michel, Yotov, Chaabouni, Comanescu, Jana, Anil, McIlroy, Liu, Mullins, Smith, Borgeaud, Girgin, Douglas, Pandya, Shakeri, De, Klimenko, Hennigan, Feinberg, Stokowiec, hui Chen, Ahmed, Gong, Warkentin, Peran, Giang, Farabet, Vinyals, Dean, Kavukcuoglu, Hassabis, Ghahramani, Eck, Barral, Pereira, Collins, Joulin, Fiedel, Senter, Andreev, and Kenealy]{gemmateam2024gemmaopenmodelsbased}
Gemma Team, Thomas Mesnard, Cassidy Hardin, Robert Dadashi, Surya Bhupatiraju, Shreya Pathak, Laurent Sifre, Morgane Rivière, Mihir~Sanjay Kale, Juliette Love, Pouya Tafti, Léonard Hussenot, Pier~Giuseppe Sessa, Aakanksha Chowdhery, Adam Roberts, Aditya Barua, Alex Botev, Alex Castro-Ros, Ambrose Slone, Amélie Héliou, Andrea Tacchetti, Anna Bulanova, Antonia Paterson, Beth Tsai, Bobak Shahriari, Charline~Le Lan, Christopher~A. Choquette-Choo, Clément Crepy, Daniel Cer, Daphne Ippolito, David Reid, Elena Buchatskaya, Eric Ni, Eric Noland, Geng Yan, George Tucker, George-Christian Muraru, Grigory Rozhdestvenskiy, Henryk Michalewski, Ian Tenney, Ivan Grishchenko, Jacob Austin, James Keeling, Jane Labanowski, Jean-Baptiste Lespiau, Jeff Stanway, Jenny Brennan, Jeremy Chen, Johan Ferret, Justin Chiu, Justin Mao-Jones, Katherine Lee, Kathy Yu, Katie Millican, Lars~Lowe Sjoesund, Lisa Lee, Lucas Dixon, Machel Reid, Maciej Mikuła, Mateo Wirth, Michael Sharman, Nikolai Chinaev, Nithum Thain, Olivier Bachem,
  Oscar Chang, Oscar Wahltinez, Paige Bailey, Paul Michel, Petko Yotov, Rahma Chaabouni, Ramona Comanescu, Reena Jana, Rohan Anil, Ross McIlroy, Ruibo Liu, Ryan Mullins, Samuel~L Smith, Sebastian Borgeaud, Sertan Girgin, Sholto Douglas, Shree Pandya, Siamak Shakeri, Soham De, Ted Klimenko, Tom Hennigan, Vlad Feinberg, Wojciech Stokowiec, Yu~hui Chen, Zafarali Ahmed, Zhitao Gong, Tris Warkentin, Ludovic Peran, Minh Giang, Clément Farabet, Oriol Vinyals, Jeff Dean, Koray Kavukcuoglu, Demis Hassabis, Zoubin Ghahramani, Douglas Eck, Joelle Barral, Fernando Pereira, Eli Collins, Armand Joulin, Noah Fiedel, Evan Senter, Alek Andreev, and Kathleen Kenealy.
\newblock Gemma: Open models based on gemini research and technology, 2024.
\newblock URL \url{https://arxiv.org/abs/2403.08295}.

\bibitem[Vaswani et~al.(2017)Vaswani, Shazeer, Parmar, Uszkoreit, Jones, Gomez, Kaiser, and Polosukhin]{vaswani2017attention}
Ashish Vaswani, Noam Shazeer, Niki Parmar, Jakob Uszkoreit, Llion Jones, Aidan~N Gomez, {\L}ukasz Kaiser, and Illia Polosukhin.
\newblock Attention is all you need.
\newblock \emph{Advances in neural information processing systems}, 30, 2017.

\bibitem[Veličković et~al.(2025)Veličković, Perivolaropoulos, Barbero, and Pascanu]{softmaxpetar}
Petar Veličković, Christos Perivolaropoulos, Federico Barbero, and Razvan Pascanu.
\newblock Softmax is not enough (for sharp size generalisation), 2025.
\newblock URL \url{https://arxiv.org/abs/2410.01104}.

\bibitem[Wertheimer et~al.(2026)Wertheimer, Zhang, Liu, Yin, and Wang]{wertheimer2026frayed}
Davis Wertheimer, Aozhong Zhang, Derrick Liu, Penghang Yin, and Naigang Wang.
\newblock Frayed rope and long inputs: A geometric perspective.
\newblock \emph{arXiv preprint arXiv:2603.18017}, 2026.

\bibitem[Xu et~al.(2024)Xu, Men, Wang, Zhang, Lin, Lu, Han, and Chen]{xu2024base}
Mingyu Xu, Xin Men, Bingning Wang, Qingyu Zhang, Hongyu Lin, Yaojie Lu, Xianpei Han, and Weipeng Chen.
\newblock Base of rope bounds context length.
\newblock \emph{Advances in Neural Information Processing Systems}, 37:\penalty0 87386--87410, 2024.

\bibitem[Zadouri et~al.(2026)Zadouri, Hoehnerbach, Shah, Liu, Thakkar, and Dao]{zadouri2026flashattention4algorithmkernelpipelining}
Ted Zadouri, Markus Hoehnerbach, Jay Shah, Timmy Liu, Vijay Thakkar, and Tri Dao.
\newblock Flashattention-4: Algorithm and kernel pipelining co-design for asymmetric hardware scaling, 2026.
\newblock URL \url{https://arxiv.org/abs/2603.05451}.

\end{thebibliography}
\newpage
\appendix

\providecolor{Green}{RGB}{0,128,0}
\providecolor{Red}{RGB}{200,0,0}
\providecommand{\cmark}{\textcolor{Green}{\ding{51}}}
\providecommand{\xmark}{\textcolor{Red}{\ding{55}}}

\sisetup{
    mode=text,
    table-alignment-mode=none,
    reset-text-family = false,
    reset-text-series = false,
    reset-text-shape = false,
    table-number-alignment = center,
    table-align-comparator = false,
    table-align-text-pre = false,
    table-align-text-post = false,
    round-mode=uncertainty,
    round-precision=3
}

\ifdefined\ifgapetheorembars\else
  
\fi


\newenvironment{restatedtheorem}[2][]{%
  \begin{gapestatementframe}{theorembar}
  \noindent\textbf{Theorem~\ref{#2}}%
  \if\relax\detokenize{#1}\relax
  .%
  \else
  \textbf{ (#1).}%
  \fi
  \itshape
}{%
  \end{gapestatementframe}
}

\newenvironment{restatedproposition}[2][]{%
  \begin{gapestatementframe}{propositionbar}
  \noindent\textbf{Proposition~\ref{#2}}%
  \if\relax\detokenize{#1}\relax
  .%
  \else
  \textbf{ (#1).}%
  \fi
  \itshape
}{%
  \end{gapestatementframe}
}

\newenvironment{restatedlemma}[2][]{%
  \begin{gapestatementframe}{lemmabar}
  \noindent\textbf{Lemma~\ref{#2}}%
  \if\relax\detokenize{#1}\relax
  .%
  \else
  \textbf{ (#1).}%
  \fi
  \itshape
}{%
  \end{gapestatementframe}
}

\newenvironment{restatedcorollary}[2][]{%
  \begin{gapestatementframe}{corollarybar}
  \noindent\textbf{Corollary~\ref{#2}}%
  \if\relax\detokenize{#1}\relax
  .%
  \else
  \textbf{ (#1).}%
  \fi
  \itshape
}{%
  \end{gapestatementframe}
}

\providecommand*{\theoremautorefname}{Theorem}
\providecommand*{\lemmaautorefname}{Lemma}
\providecommand*{\propositionautorefname}{Proposition}
\providecommand*{\definitionautorefname}{Definition}
\providecommand*{\corollaryautorefname}{Corollary}
\providecommand*{\remarkautorefname}{Remark}
\providecommand*{\exampleautorefname}{Example}
\providecommand*{\assumptionautorefname}{Assumption}

\makeatletter
\@ifundefined{ifshowcomments}{\newboolean{showcomments}}{}
\makeatother
\setboolean{showcomments}{true}
\providecommand{\fr}[1]{\textcolor{red}{\textbf{[FR: #1]}}}

\definecolor{Top1}{HTML}{0072B2}
\definecolor{Top2}{HTML}{E69F00}
\definecolor{Top3}{HTML}{CC79A7}

\usetikzlibrary{calc,positioning}

\definecolor{ufill}{HTML}{E1D5E7}
\definecolor{uborder}{HTML}{9673A6}
\definecolor{vfill}{HTML}{D5E8D4}
\definecolor{vborder}{HTML}{82B366}

\begin{center}
    {\LARGE \bfseries Remember to Forget:\\Gated Adaptive Positional
Encoding \par}
    \vspace{0.4em}
    {\Large Supplementary Material\par}
\end{center}

\vspace{1em}

\addcontentsline{toc}{part}{Supplementary Material}
\etocsetnexttocdepth{3}
\localtableofcontents

\newpage
\section{Gated Adaptive Positional Encoding: Proofs and Further Explanations}
\label{app:A}

In this section, we provide the proofs omitted from the theoretical results presented in Section~\ref{sec:GAPE} and Section~\ref{sec:selective_routing}. We also give additional details for several properties of the GAPE structural bias that are only stated in the main text.

\subsection{Landmark Protection}
\label{app:selective_routing}

We restate Theorem~\ref{thm:bilateral_invariant} and provide its proof.
\vspace{0.5em}

\begin{restatedtheorem}
{thm:bilateral_invariant}
Let $j_{\mathrm{distant}} \ll i$. If $l_{j_{\mathrm{distant}}}=1$, then
$M_{i,i}=M_{i,j_{\mathrm{distant}}}$; more generally,
$M_{i,j_{\mathrm{distant}}}\to M_{i,i}$ as
$l_{j_{\mathrm{distant}}}\to1$.
Moreover, for any $a<b<i$ with $l_b<1$, there exists a
threshold $l_a^\ast\in(0,1)$ such that for all
$l_a\in(l_a^\ast,1]$: $M_{i,a}>M_{i,b}$.
\end{restatedtheorem}

\begin{proof}
Recall that:
$$
M_{i,j}
=
\frac{\Gamma_h g_i}{T}
\big[j(1-l_j)+il_j\big]
$$
For the current token $j=i$, independently of the value of $l_i$, we have:
$$
M_{i,i}
=
\frac{\Gamma_h g_i}{T}
\big[i(1-l_i)+il_i\big]
=
\frac{\Gamma_h g_i}{T}i
$$
For a distant key with full landmark protection, $l_{j_{\mathrm{distant}}}=1$:
$$
M_{i,j_{\mathrm{distant}}}
=
\frac{\Gamma_h g_i}{T}
\big[j_{\mathrm{distant}}(1-1)+i(1)\big]
=
\frac{\Gamma_h g_i}{T}i
$$
and therefore $M_{i,i}=M_{i,j_{\mathrm{distant}}}$.
Since $M_{i,j}$ is continuous in $l_j$, the same calculation gives
$M_{i,j_{\mathrm{distant}}}\to M_{i,i}$ as
$l_{j_{\mathrm{distant}}}\to1$.

For the second claim, let $a<b<i$ and $l_b<1$. Then:
$$
M_{i,a}-M_{i,b}
=
\frac{\Gamma_h g_i}{T}
\Big[(a-b)+(i-a)l_a-(i-b)l_b\Big]
$$
Solving $M_{i,a}>M_{i,b}$ for $l_a$ yields:
$$
l_a
>
\frac{(i-b)l_b+(b-a)}{i-a}
\equiv l_a^\ast
$$
Because $a<b<i$ and $0\le l_b<1$, we have:
$$
l_a^\ast>0
$$
and:
$$
(i-b)l_b+(b-a)
<
(i-b)+(b-a)
=
i-a
$$
so $l_a^\ast<1$. Hence $l_a^\ast\in(0,1)$, and any
$l_a>l_a^\ast$ makes the older key $a$ receive a larger structural bias
than the newer key $b$.
\end{proof}

\subsection{Modulation of Recency Bias}
\label{app:recency_bias}

We restate Theorem~\ref{thm:asymptotic_eradication} and provide its proof.
\vspace{0.5em}

\begin{restatedtheorem}{thm:asymptotic_eradication}
Let $P_i$ and $U_i$ denote the protected and unprotected sets introduced in
Section~\ref{sec:GAPE}, and let
$S_{\max}:=\max_m\|\mathbf q_i\|\,\|\mathbf k_m\|$.
Then, the total attention mass on $U_i$ satisfies:
\begin{equation}
    \sum_{k \in U_i} \alpha_{i,k}
\;\le\;
e^{2S_{\max}}
\sum_{k \in U_i}
\exp\!\left(
-\frac{\Gamma_h g_i}{T}(i-k)(1-l_k)
\right)
\end{equation}
If, in addition, the keys under consideration satisfy
$i-k\ge\Delta_{\min}$ and remain uniformly separated from full protection,
$l_k\le1-\epsilon_l$ for some $\epsilon_l>0$, then:
\begin{equation}
    \sum_{k \in U_i} \alpha_{i,k}
\;\le\;
|U_i|\,\exp\!\left(
2S_{\max}
-\frac{\epsilon_l\Gamma_h g_i}{T}\Delta_{\min}
\right)
\end{equation}
and their total attention mass vanishes exponentially as $g_i\to\infty$.
\end{restatedtheorem}

\begin{proof}
By Cauchy--Schwarz and the isometric property of $R_\Theta$, we have:
$$
|s_{i,m}|
=
\frac{1}{\sqrt d}
|\mathbf q_i^\top R_\Theta \mathbf k_m|
\le
\frac{1}{\sqrt d}
\|\mathbf q_i\|\,\|\mathbf k_m\|
\le
\|\mathbf q_i\|\,\|\mathbf k_m\|
\le
S_{\max}
$$
for every $m$.

For any key $k\in U_i$, the softmax probability satisfies:
$$
\alpha_{i,k}
=
\frac{\exp(a_{i,k})}{Z_i},
\qquad
Z_i
=
\sum_{m=0}^{i}\exp(a_{i,m})
$$
Since $Z_i\ge\exp(a_{i,i})$:
$$
\alpha_{i,k}
\le
\frac{\exp(a_{i,k})}{\exp(a_{i,i})}
=
\exp(s_{i,k}-s_{i,i})
\exp(M_{i,k}-M_{i,i})
$$

Using the GAPE mask:
$$
M_{i,k}
=
\frac{\Gamma_h g_i}{T}
\big[k(1-l_k)+il_k\big]
$$
while:
$$
M_{i,i}
=
\frac{\Gamma_h g_i}{T}i
$$
Therefore:
\begin{align*}
M_{i,k}-M_{i,i}
&=
\frac{\Gamma_h g_i}{T}
\big[k(1-l_k)+il_k-i\big] \\
&=
-\frac{\Gamma_h g_i}{T}
(i-k)(1-l_k)
\end{align*}
and hence:
$$
\alpha_{i,k}
\le
\exp(s_{i,k}-s_{i,i})
\exp\!\left(
-\frac{\Gamma_h g_i}{T}
(i-k)(1-l_k)
\right)
$$
Since $|s_{i,k}|,|s_{i,i}|\le S_{\max}$:
$$
s_{i,k}-s_{i,i}\le2S_{\max}
$$
so:
$$
\alpha_{i,k}
\le
e^{2S_{\max}}
\exp\!\left(
-\frac{\Gamma_h g_i}{T}
(i-k)(1-l_k)
\right)
$$
Summing over $k\in U_i$ gives:
$$
\sum_{k\in U_i}\alpha_{i,k}
\le
e^{2S_{\max}}
\sum_{k\in U_i}
\exp\!\left(
-\frac{\Gamma_h g_i}{T}
(i-k)(1-l_k)
\right)
$$

Suppose now that every key in the set satisfies:
$$
i-k\ge\Delta_{\min},
\qquad
l_k\le1-\epsilon_l
$$
so that $1-l_k\ge\epsilon_l$. Then:
$$
(i-k)(1-l_k)
\ge
\epsilon_l\Delta_{\min}
$$
and therefore:
$$
\alpha_{i,k}
\le
e^{2S_{\max}}
\exp\!\left(
-\frac{\epsilon_l\Gamma_h g_i}{T}
\Delta_{\min}
\right)
$$
Summing over the set yields:
$$
\sum_{k\in U_i}\alpha_{i,k}
\le
|U_i|\,
\exp\!\left(
2S_{\max}
-\frac{\epsilon_l\Gamma_h g_i}{T}
\Delta_{\min}
\right)
$$
which converges exponentially to zero as $g_i\to\infty$.
\end{proof}

\subsection{Softmax Normalization Growth}
\label{app:partition_growth}

A separate difficulty in long-context attention is the growth of the softmax
normalization term as additional background tokens enter the context. The
previous result shows that GAPE attenuates keys according to their
landmark-adjusted distance $(i-k)(1-l_k)$. Under a uniform separation from
full landmark protection, this attenuation also controls the cumulative
contribution of distant background tokens to the partition function.

\vspace{0.5em}

\begin{lemma}
\label{lem:partition_growth}
Let:
\begin{equation}
\label{eq:partition_growth_Z}
Z_i
=
\sum_{m=0}^{i}
\exp\!\bigl(s_{i,m}+M_{i,m}\bigr)
\end{equation}
denote the softmax normalization term at decoding position $i$. Define:
\begin{equation}
\label{eq:partition_growth_C}
C_i
=
\frac{\Gamma_h g_i}{T},
\qquad
C_{\min}
=
\frac{\Gamma_h \epsilon_g}{T}
\end{equation}
and assume:
\begin{equation}
\label{eq:partition_growth_scores}
s_{i,m}\le S_{\max},
\qquad
g_i\ge\epsilon_g>0
\end{equation}

Let $P_i$ denote a set of fully protected keys, with $l_j=1$ for
$j\in P_i$, together with the current position $i$, and let $U_i$ contain
the remaining keys. Assume:
\begin{equation}
\label{eq:partition_growth_pmax}
|P_i|
\le
P_{\max}
<
\infty
\end{equation}
uniformly in $i$, and suppose that for some $\epsilon_l>0$:
\begin{equation}
\label{eq:partition_growth_separation}
l_k
\le
1-\epsilon_l
\qquad
\text{for every } k\in U_i
\end{equation}

Then:
\begin{equation}
\label{eq:partition_growth_bound}
Z_i
\le
e^{S_{\max}+C_i i}
\left(
P_{\max}
+
\frac{1}{e^{C_{\min}\epsilon_l}-1}
\right)
\end{equation}

Thus, relative to the common structural level of fully protected keys, the
cumulative contribution of keys uniformly separated from full protection
remains bounded independently of the context length.
\end{lemma}

\begin{proof}
For every fully protected key $j\in P_i$:
$$
M_{i,j}
=
C_i i
$$
The same identity holds for the current position $j=i$, independently of
$l_i$.

For every $k\in U_i$:
\begin{align*}
M_{i,k}
&=
C_i\big[k(1-l_k)+il_k\big] \\
&=
C_i\big[i-(i-k)(1-l_k)\big]
\end{align*}
Since $1-l_k\ge\epsilon_l$:
$$
M_{i,k}
\le
C_i i
-
C_i\epsilon_l(i-k)
$$

Therefore:
\begin{align*}
Z_i
&=
\sum_{j\in P_i}
\exp(s_{i,j}+C_i i)
+
\sum_{k\in U_i}
\exp(s_{i,k}+M_{i,k}) \\
&\le
e^{S_{\max}+C_i i}|P_i|
+
e^{S_{\max}+C_i i}
\sum_{k\in U_i}
e^{-C_i\epsilon_l(i-k)}
\end{align*}
and hence:
$$
Z_i
\le
e^{S_{\max}+C_i i}
\left(
|P_i|
+
\sum_{k=0}^{i-1}
e^{-C_i\epsilon_l(i-k)}
\right)
$$
Setting $r=i-k$ gives:
$$
\sum_{k=0}^{i-1}
e^{-C_i\epsilon_l(i-k)}
=
\sum_{r=1}^{i}
e^{-C_i\epsilon_l r}
\le
\sum_{r=1}^{\infty}
e^{-C_i\epsilon_l r}
=
\frac{1}{e^{C_i\epsilon_l}-1}
$$
Since $C_i\ge C_{\min}$:
$$
\frac{1}{e^{C_i\epsilon_l}-1}
\le
\frac{1}{e^{C_{\min}\epsilon_l}-1}
$$
Combining this with $|P_i|\le P_{\max}$ yields:
$$
Z_i
\le
e^{S_{\max}+C_i i}
\left(
P_{\max}
+
\frac{1}{e^{C_{\min}\epsilon_l}-1}
\right)
$$
which proves the claim.
\end{proof}

\subsection{Entropy Contraction}
\label{app:entropy_contraction}

The previous results characterize the attenuation of weakly protected keys.
We now consider the idealized saturated regime in which a subset of keys is
fully protected and ask how the attention distribution behaves as the query
gate increases.

\vspace{0.5em}

\begin{corollary}
\label{cor:entropy_collapse}
Fix a query position $i$ and partition its context into a protected set
$P_i$ and a remaining set $U_i$. Assume that:
\[
l_j=1
\qquad
\text{for every } j\in P_i\setminus\{i\},
\]
and that for some $\epsilon_l>0$:
\[
l_k\le1-\epsilon_l
\qquad
\text{for every } k\in U_i.
\]
Let $\alpha_i^\ast$ denote the renormalized attention distribution supported
on $P_i$:
\begin{equation}
    \alpha_{i,j}^\ast =
\frac{\exp(a_{i,j})}{\sum_{m\in P_i}\exp(a_{i,m})},
\qquad j\in P_i
\end{equation}
Then, under \autoref{thm:asymptotic_eradication}, the full attention
distribution $\alpha_i$ converges in total variation to the extension of
$\alpha_i^\ast$ by zeros on $U_i$ as $g_i\to\infty$. Consequently:
\begin{equation}
    \lim_{g_i\to\infty} H(\alpha_i)
    =
    H(\alpha_i^\ast)
\end{equation}
Moreover, the relative probabilities inside $P_i$ are independent of $g_i$.
\end{corollary}

\begin{proof}
We proceed in three steps.

\textbf{Step 1.}
By \autoref{thm:asymptotic_eradication}, and since all $k\in U_i$ satisfy
$l_k\le1-\epsilon_l$:
\[
\sum_{k\in U_i}\alpha_{i,k}
\;\longrightarrow\;
0
\qquad
\text{as }g_i\to\infty
\]
Thus, asymptotically, all attention mass is concentrated on $P_i$.

\textbf{Step 2.}
Every $j\in P_i$ lies at the same structural level: for $j\neq i$ this follows from full landmark protection, while for $j=i$ it holds independently of $l_i$. Hence:
$$
M_{i,j}
=
\frac{\Gamma_h g_i}{T}i
$$
Hence, for any $j,m\in P_i$:
$$
a_{i,j}-a_{i,m}
=
s_{i,j}-s_{i,m}
$$
because the structural contribution cancels. It follows that:
\[
\frac{\alpha_{i,j}}
{\sum_{m\in P_i}\alpha_{i,m}}
=
\frac{\exp(s_{i,j})}
{\sum_{m\in P_i}\exp(s_{i,m})}
=
\alpha_{i,j}^\ast
\]
so the conditional distribution on $P_i$ is independent of $g_i$.

\textbf{Step 3.}
Let $\bar{\alpha}_i^\ast$ denote the extension of $\alpha_i^\ast$ by zeros
on $U_i$. Then:
\[
\|\alpha_i-\bar{\alpha}_i^\ast\|_{\mathrm{TV}}
=
\frac{1}{2}
\sum_{j=0}^{i}
\left|
\alpha_{i,j}-\bar{\alpha}_{i,j}^\ast
\right|
\]
Splitting the sum over $P_i$ and $U_i$:
\[
\|\alpha_i-\bar{\alpha}_i^\ast\|_{\mathrm{TV}}
=
\frac{1}{2}
\left(
\sum_{j\in P_i}
\left|
\alpha_{i,j}-\alpha_{i,j}^\ast
\right|
+
\sum_{k\in U_i}\alpha_{i,k}
\right)
\]
For $j\in P_i$:
\[
\alpha_{i,j}
=
\alpha_{i,j}^\ast
\sum_{m\in P_i}\alpha_{i,m},
\]
and therefore:
\[
\left|
\alpha_{i,j}-\alpha_{i,j}^\ast
\right|
=
\alpha_{i,j}^\ast
\left|
\sum_{m\in P_i}\alpha_{i,m}-1
\right|
\]
Summing over $j\in P_i$ gives:
\[
\sum_{j\in P_i}
\left|
\alpha_{i,j}-\alpha_{i,j}^\ast
\right|
=
\left|
\sum_{m\in P_i}\alpha_{i,m}-1
\right|
=
\sum_{k\in U_i}\alpha_{i,k}
\]
which vanishes by Step~1. Hence:
\[
\|\alpha_i-\bar{\alpha}_i^\ast\|_{\mathrm{TV}}
\;\longrightarrow\;
0
\]
as $g_i\to\infty$.

For fixed $i$, the attention distribution lies on a finite-dimensional
probability simplex, on which Shannon entropy is continuous. Therefore:
\[
H(\alpha_i)
\;\longrightarrow\;
H(\bar{\alpha}_i^\ast)
=
H(\alpha_i^\ast)
\]
which proves the claim.
\end{proof}

\subsection{Hardware-Native Factorisation}
\label{app:hardware_factorisation}

We restate Proposition~\ref{prop:rank2_factorization} and provide its proof.
\vspace{0.5em}

\begin{restatedproposition}{prop:rank2_factorization}
Let $T$ denote the training context length and $d$ the full attention head
dimension. The GAPE mask $M_{i,j}$ and the mask
$\widehat{M}_{i,j}$ are equivalent post-softmax. Furthermore,
$M_{i,j}$ can be computed within standard fused Scaled Dot-Product
Attention (SDPA) without materializing an explicit $N\times N$ bias matrix.
\end{restatedproposition}

\begin{proof}
\textbf{Step 1: Post-softmax equivalence of $\widehat{M}_{i,j}$ and
$M_{i,j}$.}
We first show that the two masks differ only by an additive term depending
on the query index $i$. Direct computation gives:
\begin{align*}
    \widehat{M}_{i,j} - M_{i,j}
    &= -\frac{\Gamma_h g_i}{T}
    \bigl[(1-l_j)(i-j)+j(1-l_j)+i\,l_j\bigr] \\
    &= -\frac{\Gamma_h g_i}{T}
    \bigl[(1-l_j)\bigl((i-j)+j\bigr)+i\,l_j\bigr] \\
    &= -\frac{\Gamma_h g_i}{T}
    \bigl[(1-l_j)i+i\,l_j\bigr] \\
    &= -\frac{\Gamma_h g_i}{T}i
    \;=:\;c_i
\end{align*}
Hence:
\begin{equation*}
    \bigl(s_{i,j}+\widehat{M}_{i,j}\bigr)
    -
    \bigl(s_{i,j}+M_{i,j}\bigr)
    =
    c_i,
    \qquad
    \text{independent of }j.
\end{equation*}
Because softmax is invariant to an additive shift shared by all keys for a
fixed query:
\begin{equation*}
    \frac{\exp\!\bigl(s_{i,j}+\widehat{M}_{i,j}\bigr)}
    {\sum_{j'}\exp\!\bigl(s_{i,j'}+\widehat{M}_{i,j'}\bigr)}
    =
    \frac{\exp\!\bigl(s_{i,j}+M_{i,j}\bigr)}
    {\sum_{j'}\exp\!\bigl(s_{i,j'}+M_{i,j'}\bigr)}.
\end{equation*}
Thus, the two masks induce exactly the same attention distribution and
attention output.

\textbf{Step 2: Fused-SDPA realisation of $M_{i,j}$.}
Let $\mathbf q_i,\mathbf k_j\in\mathbb R^{d-2}$ denote the semantic
coordinates after the chosen multiplicative positional map, so that their
scaled inner product gives $s_{i,j}$. We reserve the final two coordinates
of the $d$-dimensional head space for the structural pathway and define:
$$
\tilde{\mathbf{q}}_i =
    \begin{bmatrix}
    \mathbf{q}_i \\
    \Gamma_h g_i \sqrt{d} \\
    \Gamma_h g_i \left(\frac{i}{T}\right)\sqrt{d}
    \end{bmatrix},
    \qquad
\tilde{\mathbf{k}}_j =
    \begin{bmatrix}
    \mathbf{k}_j \\
    \dfrac{j(1-l_j)}{T} \\
    l_j
    \end{bmatrix}
$$
Under standard SDPA scaling:
\begin{align*}
    a_{i,j}
    &=
    \frac{1}{\sqrt{d}}
    \tilde{\mathbf q}_i^\top\tilde{\mathbf k}_j \\
    &=
    \frac{1}{\sqrt{d}}
    \mathbf q_i^\top\mathbf k_j
    +
    \frac{\Gamma_h g_i}{T}
    \bigl[j(1-l_j)+i\,l_j\bigr] \\
    &=
    s_{i,j}+M_{i,j}.
\end{align*}
The structural bias is therefore represented by two separable query--key
terms and constitutes a rank-at-most-two correction to the logit matrix.
It is evaluated inside the same scaled dot product as the semantic term,
without materializing a separate $N\times N$ bias matrix. Consequently,
the factorized GAPE representation can be evaluated with fused causal SDPA
implementations such as FlashAttention.
\end{proof}

\subsection{Translation Invariance}
\label{app:translation_invariance}

The original relative-distance form $\widehat M_{i,j}$ depends on position
only through $i-j$, making its translation invariance immediate whenever
the semantic scores and learned gates themselves are translation-invariant.
The factorized mask $M_{i,j}$ differs from this representation by a
query-dependent constant and therefore induces the same invariant attention
distribution.

\vspace{0.5em}

\begin{proposition}
\label{prop:translation_invariance}
Let the input sequence be shifted by a constant offset $\Delta>0$, so that
$i'=i+\Delta$ and $j'=j+\Delta$. Assume that the semantic scores and gates
are unchanged by this shift:
$s_{i',j'}=s_{i,j}$, $g_{i'}=g_i$, and $l_{j'}=l_j$.
Then:
\[
\alpha_{i',j'}
=
\alpha_{i,j}
\]
for all $j\le i$.
\end{proposition}

\begin{proof}
Let:
$$
C
=
\frac{\Gamma_h g_i}{T}.
$$
Then:
\begin{align*}
    M_{i',j'}
    &=
    C\Bigl[(j+\Delta)(1-l_j)+(i+\Delta)l_j\Bigr] \\
    &=
    C\Bigl[j(1-l_j)+il_j\Bigr]+C\Delta \\
    &=
    M_{i,j}+C\Delta.
\end{align*}
Hence every shifted logit for the same query receives the same additive
offset:
$$
a_{i',j'}
=
s_{i',j'}+M_{i',j'}
=
s_{i,j}+M_{i,j}+C\Delta.
$$
This query-dependent constant cancels exactly under softmax normalization,
and therefore:
$$
\alpha_{i',j'}
=
\alpha_{i,j}.
$$
\end{proof}

\subsection{Mitigating Phase Hallucinations}
\label{app:bounding_hallucinations}

We restate and extend \autoref{thm:dynamic_and_asymptotic}, and provide its
proof. The result characterizes how the GAPE structural term can prevent a
weakly protected key from dominating a reference logit even when its
phase-modulated semantic score is maximally favorable.

\vspace{0.5em}

\begin{restatedtheorem}{thm:dynamic_and_asymptotic}
Assume that the semantic scores are uniformly bounded as
$|s_{i,m}|\le S_{\max}$ for all $m$, and let $g_i>0$.
For any key $j<i$ with $l_j<1$, let $\Delta=i-j$. Then:
\begin{equation}
    a_{i,i}-a_{i,j}
    \ge
    -2S_{\max}
    +
    \frac{\Gamma_h g_i}{T}
    \Delta(1-l_j)
\end{equation}
Consequently, a sufficient condition for the current-position logit to
strictly exceed that of key $j$ is:
\begin{equation}
    g_i
    >
    \frac{2S_{\max}T}
    {\Gamma_h\Delta(1-l_j)}
\end{equation}
For a fixed key $k<i$ with $l_k<1$, the corresponding sufficient gate
threshold is:
\begin{equation}
    g_{\min}^{(k)}(i)
    =
    \frac{2S_{\max}T}
    {\Gamma_h(i-k)(1-l_k)}
\end{equation}
and, if $l_k$ remains fixed below $1$ as $i$ increases:
\begin{equation}
    \lim_{i\to\infty}
    g_{\min}^{(k)}(i)
    =
    0.
\end{equation}
More generally, the same limit holds whenever
$1-l_k\ge\epsilon_l>0$ uniformly in $i$.
\end{restatedtheorem}

\begin{proof}
For the current position $j=i$:
$$
M_{i,i}
=
\frac{\Gamma_h g_i}{T}i,
$$
independently of the value of $l_i$.

For any earlier key $j<i$:
\begin{align*}
M_{i,j}
&=
\frac{\Gamma_h g_i}{T}
\big[j(1-l_j)+il_j\big] \\
&=
\frac{\Gamma_h g_i}{T}
\big[i-(i-j)(1-l_j)\big].
\end{align*}
Hence:
$$
M_{i,i}-M_{i,j}
=
\frac{\Gamma_h g_i}{T}
(i-j)(1-l_j).
$$
Therefore:
\begin{align*}
a_{i,i}-a_{i,j}
&=
(s_{i,i}-s_{i,j})
+
(M_{i,i}-M_{i,j}) \\
&=
(s_{i,i}-s_{i,j})
+
\frac{\Gamma_h g_i}{T}
(i-j)(1-l_j).
\end{align*}
Since:
$$
s_{i,i}-s_{i,j}
\ge
-2S_{\max},
$$
we obtain:
$$
a_{i,i}-a_{i,j}
\ge
-2S_{\max}
+
\frac{\Gamma_h g_i}{T}
\Delta(1-l_j).
$$

The right-hand side is strictly positive whenever:
$$
g_i
>
\frac{2S_{\max}T}
{\Gamma_h\Delta(1-l_j)}.
$$
Thus this condition is sufficient to ensure that the considered key cannot
exceed the current-position logit, regardless of its semantic score within
the assumed bound.

For a fixed key $k$, setting $\Delta=i-k$ gives:
$$
g_{\min}^{(k)}(i)
=
\frac{2S_{\max}T}
{\Gamma_h(i-k)(1-l_k)}.
$$
If $l_k<1$ remains fixed, then $(i-k)(1-l_k)\to\infty$ as
$i\to\infty$, and therefore:
$$
\lim_{i\to\infty}
g_{\min}^{(k)}(i)
=
0.
$$
The same conclusion follows if $1-l_k$ is uniformly bounded below by some
$\epsilon_l>0$.
\end{proof}



\subsection{GAPE Theoretical Properties in Needle-In-A-Haystack Retrieval}
\label{app:niah_retrieval}

We next specialize the same argument to the controlled NIAH setting.
The purpose of this analysis is not to characterize end-to-end task
accuracy, which depends on the full network, but to isolate the pairwise
attention condition under which a historical needle can outrank the current
position in the attention logits. This gives a simple analytical view of
how semantic similarity, physical distance, query contraction, and
landmark protection interact.

\vspace{0.5em}

\begin{proposition}
\label{prop:niah_threshold}
Let $k<i$ denote the absolute position of a target needle token within a
sequence ending at query position $i$, and let
$\rho=k/i\in(0,1)$ denote its relative depth. The historical needle receives
a larger attention logit than the current position,
$a_{i,k}>a_{i,i}$, if and only if:
\begin{equation}
    s_{i,k} - s_{i,i}
    >
    \frac{\Gamma_h g_i}{T}\,
    i(1-\rho)(1-l_k)
\end{equation}
\end{proposition}

\begin{proof}
Using:
\[
a_{i,k}>a_{i,i},
\qquad
a_{i,m}=s_{i,m}+M_{i,m},
\]
we obtain:
\[
s_{i,k}-s_{i,i}
>
M_{i,i}-M_{i,k}.
\]
Substituting:
\[
M_{i,i}
=
\frac{\Gamma_h g_i}{T}i,
\qquad
M_{i,k}
=
\frac{\Gamma_h g_i}{T}
\bigl[
k(1-l_k)+il_k
\bigr]
\]
gives:
\begin{align*}
M_{i,i}-M_{i,k}
&=
\frac{\Gamma_h g_i}{T}
\bigl[
i-k(1-l_k)-il_k
\bigr] \\
&=
\frac{\Gamma_h g_i}{T}
(i-k)(1-l_k).
\end{align*}
Finally, substituting $k=\rho i$ yields:
\[
M_{i,i}-M_{i,k}
=
\frac{\Gamma_h g_i}{T}
i(1-\rho)(1-l_k),
\]
which proves the claim.
\end{proof}

\begin{figure}[htbp]
    \centering
    \includegraphics[width=\linewidth]{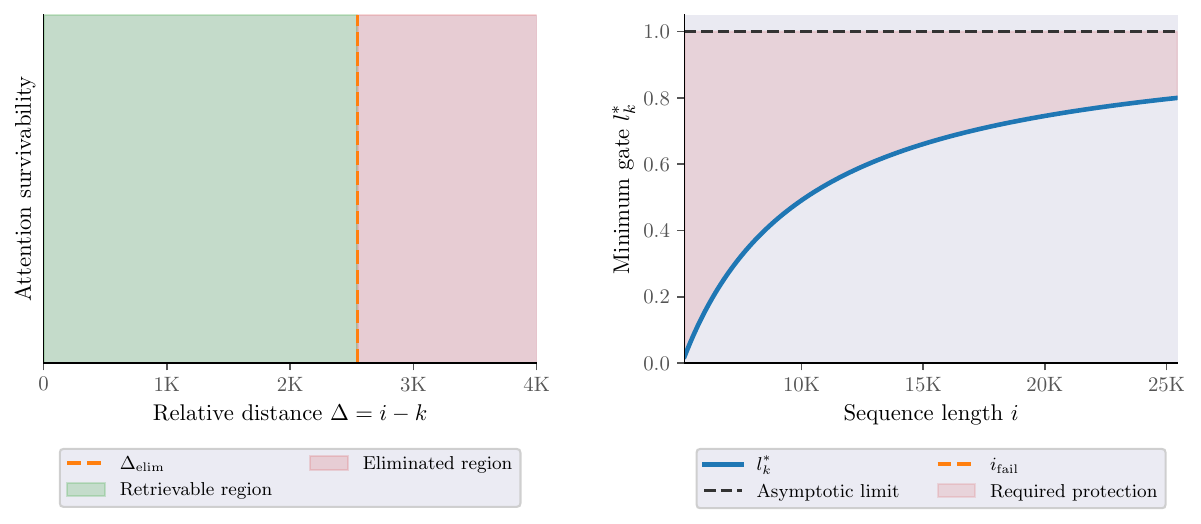}
    \caption{
    \emph{Visualisation of the GAPE NIAH pairwise threshold.}
    \textbf{(a)} For a weakly protected needle, the maximum semantic
    advantage required to outrank the current position increases with its
    landmark-adjusted distance. Beyond the corresponding dominance horizon,
    bounded semantic similarity alone cannot overcome the structural
    penalty.
    \textbf{(b)} At fixed relative depth $\rho$, the landmark activation
    required to preserve the possibility of pairwise dominance increases
    with sequence length and approaches full protection asymptotically.
    }
    \label{fig:prop3_desmos}
\end{figure}

\noindent\textbf{Interpretation.}
\autoref{fig:prop3_desmos} illustrates the threshold in
\autoref{prop:niah_threshold}. For a fully unprotected needle, $l_k=0$, the
structural disadvantage relative to the current position grows with both
the sequence length $i$ and its relative distance $(1-\rho)$. As landmark
activation increases, this disadvantage is reduced continuously by the
factor $(1-l_k)$. The result therefore makes explicit why long-range
retrieval creates two distinct demands on GAPE: the query gate controls the
overall strength of contraction, while the landmark gate determines whether
a particular distant key remains competitive under that contraction.

\vspace{1em}

\begin{corollary}
\label{cor:eviction_horizons}
Assume that the phase-modulated semantic similarity is bounded,
$|s_{i,m}|\le S_{\max}$, so that:
\[
s_{i,k}-s_{i,i}
\le
2S_{\max}.
\]
Let the query gate satisfy $g_i\ge\epsilon_g>0$. For a key with $l_k<1$,
pairwise dominance over the current-position logit is impossible whenever:
\begin{equation}
    i-k
    >
    \Delta_{\mathrm{dom}}(l_k)
    :=
    \frac{2S_{\max}T}
    {\Gamma_h\epsilon_g(1-l_k)}
\end{equation}
For a fully unprotected needle ($l_k=0$) at relative depth
$\rho\in(0,1)$ and a fixed query-gate value $g_i$, the same pairwise
condition cannot be satisfied whenever:
\begin{equation}
    i
    \ge
    i_{\mathrm{dom}}
    :=
    \frac{2S_{\max}T}
    {\Gamma_h g_i(1-\rho)}
\end{equation}
More generally, preserving the possibility that a needle at depth $\rho$
outranks the current position requires:
\begin{equation}
    l_k
    >
    1
    -
    \frac{2S_{\max}T}
    {\Gamma_h g_i\,i(1-\rho)}
    \equiv
    l_k^\ast
\end{equation}
whenever the right-hand side lies in $[0,1]$.
\end{corollary}

\begin{proof}
By Proposition~\ref{prop:niah_threshold}, pairwise dominance requires:
$$
s_{i,k}-s_{i,i}
>
\frac{\Gamma_h g_i}{T}
(i-k)(1-l_k).
$$
Since the semantic differential is at most $2S_{\max}$, this can hold only
if:
$$
\frac{\Gamma_h g_i}{T}
(i-k)(1-l_k)
<
2S_{\max}.
$$
Using $g_i\ge\epsilon_g$ gives the necessary condition:
$$
i-k
<
\frac{2S_{\max}T}
{\Gamma_h\epsilon_g(1-l_k)}.
$$
Therefore, once the distance exceeds
$\Delta_{\mathrm{dom}}(l_k)$, bounded semantic similarity cannot make the
key outrank the current-position logit.

For a fully unprotected needle, $l_k=0$, and
$i-k=i(1-\rho)$. The condition becomes:
$$
\frac{\Gamma_h g_i}{T}
i(1-\rho)
<
2S_{\max},
$$
which is impossible whenever:
$$
i
\ge
\frac{2S_{\max}T}
{\Gamma_h g_i(1-\rho)}
=
i_{\mathrm{dom}}.
$$

Finally, solving the same necessary inequality for $l_k$ gives:
$$
l_k
>
1
-
\frac{2S_{\max}T}
{\Gamma_h g_i\,i(1-\rho)}
=
l_k^\ast,
$$
which proves the result.
\end{proof}

\newpage
\section{Extended Experiments}
\label{app:training_setup}

\subsection{nanoGPT: Training Setup and Datasets.}
\label{app:training_setup}
The model used for the experiments presented in Section~\ref{sec:experiments} is a $\sim$44.5M parameter nanoGPT model trained using the configuration in \autoref{tab:train_conf_app}; with this setup (effective batch size $64 \times 8$ and sequence length 1024 over 3000 iterations), the model is exposed to approximately $1.5 \times 10^9$ tokens. The models were trained on a machine with a single NVIDIA A100 80GB GPU.

\begin{table}[h]
\centering
\small
\caption{Training configuration for the nanoGPT model. The parameters with (*) indicate the initialization values for the GAPE configuration.}
\vspace{0.5em}
\begin{tabular}{l c}
\toprule
\textbf{Parameter} & \textbf{Value} \\
\midrule
Optimizer & AdamW \\
Layers ($n_{\text{layer}}$) & 6 \\
Heads ($n_{\text{head}}$) & 8 \\
Embedding dim ($n_{\text{embd}}$) & 512 \\
RoPE fraction & 0.75 \\
RoPE $\theta$ & 10000 \\
Batch size & 64 \\
Block size & 1024 \\
Grad. accumulation & 8 \\
Max iterations & 3000 \\
Warmup iterations & 300 \\
Learning rate & $1 \times 10^{-3}$ \\
Min learning rate & $1 \times 10^{-4}$ \\
$\beta_1, \beta_2$ & 0.9, 0.95 \\
Grad clip & 1.0 \\
Weight decay & 0.1 \\
lambda\_l1* & 0.0 \\
init\_bg* & -3.0 \\
init\_bl* & 0.0 \\
init\_gamma* & 0.5413\\
\bottomrule
\end{tabular}
\label{tab:train_conf_app}
\end{table}

\label{app:dataset}
\textbf{Datasets.} In the following, we describe the specific dataset subsets and experimental configurations utilized for the evaluations presented in Section~\ref{sec:experiments}.

\begin{itemize}
    \item \textbf{Dolma 3 Longmino Mix (50B):} We utilized three subsets of this corpus, characterized by average document token lengths of $2^{13}$, $2^{14}$, and $2^{15}$, respectively. The first subset was used for training to better align the natural document size with the training context window. The remaining two subsets were reserved exclusively for evaluation. This stratification ensures that during OOD evaluation, the models process genuinely contiguous long texts rather than artificially concatenated shorter documents.
    \item \textbf{RULER dataset:} the RULER long context benchmark dataset is presented in \citet{hsieh2024ruler}. In this setup, we use the version available on HuggingFace: \href{https://huggingface.co/datasets/rbiswasfc/ruler}{huggingface.co/datasets/rbiswasfc/ruler}.
\end{itemize}

\subsection{OLMo-3 7B Retrofit: Setup and Detailed Results}
\label{app:olmo_retrofit}

This subsection provides the implementation, adaptation configuration, and evaluation details for the OLMo-3 7B retrofit experiment in Section~\ref{sec:gape-olmo-retrofit}.

\subsubsection{Checkpoint and GAPE Implementation}
Both model variants start from the OLMo-3 7B checkpoint immediately before its long-context adaptation, \texttt{allenai/Olmo-3-1025-7B}, revision \texttt{stage2-step47684}~\cite{olmo2025olmo3}. The checkpoint was pretrained with RoPE and has a native context length of 8,192 tokens. We keep the pretrained backbone frozen and jointly train the GAPE parameters and rank-8 LoRA adapters in its eight full-attention layers. The LoRA adapters are applied to the query, key, value, and output projections. The modified layer indices are $[3,7,11,15,19,23,27,31]$. The LoRA-only control uses adapters on the same projections and layers.

The LoRA-only variant has 2,097,152 trainable parameters. GAPE adds 66,304 parameters, bringing the GAPE+LoRA variant to 2,163,456 trainable parameters, about 0.03\% of a 7B model.

GAPE is applied only to the eight full-attention layers. Its structural term is factorized into two query/key coordinates, using absolute positions normalized by the fixed value $T=12{,}288$. This implementation uses \texttt{overwrite\_reserved\_pair}: after RoPE, the structural coordinates replace the split-half RoPE pair at head dimensions 63 and 127. The head width remains 128, and we store the modified keys in the KV cache. The resulting query/key dot product incorporates the structural term without materializing a dense positional-bias matrix or requiring a custom attention kernel.

\subsubsection{Adaptation Data and Training Configuration}

We train both variants on Dolma 3 Longmino Mix (\texttt{dolma3\_longmino\_mix-50B-1025}). Sequence lengths range from 9,216 to 12,288 tokens, ensuring that all training examples exceed the checkpoint's native 8K context. Each variant is trained on 100,002,251 tokens over 8,432 optimizer steps. Table~\ref{tab:olmo_retrofit_training} summarizes the training configuration.

We choose 12,288 tokens to expose the model to positions beyond its native context while reserving longer contexts for evaluation. This is 1.5 times the native 8,192-token window. Evaluation at 4K and 8K therefore checks performance within the native range, while 16K and 32K test transfer beyond the adaptation length. The training length was not evaluated in an otherwise fixed context-length ablation.

\begin{table}[ht]
    \centering
    \caption{Adaptation configuration for the OLMo-3 7B retrofit. Both variants use the same training recipe and token budget. The query-gate initialization applies to GAPE+LoRA.}
    \label{tab:olmo_retrofit_training}
    \small
    \begin{tabular}{@{}p{0.3\linewidth}p{0.2\linewidth}@{}}
        \toprule
        \textbf{Setting} & \textbf{Value} \\
        \midrule
        LoRA target projections & \texttt{Query}, \texttt{Key}, \texttt{Value}, and \texttt{Output} \\
        LoRA rank & 8 \\
        LoRA alpha & 16 \\
        Maximum sequence length & 12,288 tokens \\
        Minimum retained chunk length & 9,216 tokens \\
        Tokens per variant & 100,002,251 \\
        Optimizer & AdamW \\
        Learning rate & $10^{-4}$ \\
        Warm-up steps & 100 \\
        Weight decay & 0 \\
        Training precision & BF16 \\
        Attention implementation & FlashAttention-2 \\
        Gradient checkpointing & Enabled \\
        Query-gate bias initialization & $b_g=-4$ \\
        \bottomrule
    \end{tabular}
\end{table}

\subsubsection{Evaluation Protocol}

\noindent\textbf{OLMES.}
We evaluate both exported models using the same OLMES Tier-1 task definitions, with five-shot evaluation. We retain the sample limits specified by the OLMES task configurations.

We report the unweighted mean score across 30 OLMES task configurations, comprising 11 non-MMLU configurations and 19 MMLU-STEM subjects.

\noindent\textbf{RULER-100.}
We evaluate all 13 RULER subtasks~\cite{hsieh2024ruler} at 4K, 8K,
16K, and 32K, using 100 examples per subtask
(RULER-100)~\cite{nvidia2026nemotron3super}. We report the unweighted
mean of the task scores at each context length. Both variants use
identical evaluation examples and greedy decoding.

\subsubsection{Detailed Results}

\noindent\textbf{General capabilities.}
Table~\ref{tab:olmo_olmes_detailed} reports the OLMES benchmark breakdown, with MMLU-STEM summarized by its mean across subjects. Across the 30 configurations, GAPE+LoRA improves 14 scores, ties 2, and decreases 14. The overall mean changes from 66.002\% to 66.056\%, a difference of $+0.054$ percentage points. The mean change over the 11 non-MMLU configurations is $+0.301$ points, while the MMLU-STEM mean changes by $-0.089$ points. These small, mixed differences give no evidence of a broad loss of performance relative to the LoRA-only adaptation on the evaluated tasks.

\begin{table}[ht]
    \centering
    \caption{OLMES results for the OLMo-3 7B retrofit. Scores are percentages; differences are GAPE+LoRA minus LoRA-only, in percentage points (pp). MC and RC denote the multiple-choice and reading-comprehension formulations. MMLU-STEM is displayed as a mean over its 19 subjects; each subject contributes separately to the overall 30-configuration mean.}
    \label{tab:olmo_olmes_detailed}
    \small
    \begin{tabular}{lrrr}
        \toprule
        \textbf{Task} & \textbf{LoRA-only} & \textbf{GAPE+LoRA} & \textbf{Difference (pp)} \\
        \midrule
        ARC Easy, MC & 93.700 & 93.900 & $+0.200$ \\
        ARC Easy, RC & 87.100 & 87.600 & $+0.500$ \\
        ARC Challenge, MC & 84.983 & 84.898 & $-0.085$ \\
        ARC Challenge, RC & 61.860 & 61.433 & $-0.427$ \\
        BoolQ & 86.000 & 86.900 & $+0.900$ \\
        CommonsenseQA & 74.120 & 74.529 & $+0.410$ \\
        HellaSwag & 78.300 & 78.900 & $+0.600$ \\
        OpenBookQA & 61.400 & 62.800 & $+1.400$ \\
        PIQA & 77.200 & 77.000 & $-0.200$ \\
        SocialIQA & 64.700 & 64.400 & $-0.300$ \\
        WinoGrande & 73.323 & 73.639 & $+0.316$ \\
        MMLU-STEM mean & 59.862 & 59.772 & $-0.089$ \\
        \midrule
        Overall mean (30 configurations) & 66.002 & 66.056 & $+0.054$ \\
        \bottomrule
    \end{tabular}
\end{table}

\noindent\textbf{Long-context tasks.}
The aggregate RULER results are reported in Table~\ref{tab:gape-olmo-retrofit}. Table~\ref{tab:olmo_ruler_detailed} reports all 13 subtasks at each context length. At 4K, all subtask differences are within one percentage point. At 8K, variable tracking decreases from 22.6\% to 13.4\%. At 16K, it improves from 15.2\% to 35.2\%, while NIAH S2 and NIAH MK2 decrease by 7.0 and 6.0 points, respectively. The effect of the retrofit therefore depends on both the subtask and context length.

At 32K, GAPE+LoRA improves 8 of 13 subtasks, ties 3, and decreases 2. The mean change is positive in each of the four task groups: $+0.94$ points for NIAH, $+0.80$ for variable tracking, $+3.42$ for aggregation, and $+2.50$ for QA. Absolute aggregate performance remains low at this length, increasing from 11.62\% to 13.17\%. The 32K result describes an improvement under extrapolation to approximately 2.67 times the adaptation length.

\begin{table}[h]
    \centering
    \caption{Complete RULER-100 results at 4K, 8K, 16K, and 32K. Scores are percentages; $\Delta$ denotes GAPE+LoRA minus LoRA-only, in percentage points. Each mean averages all 13 subtasks.}
    \label{tab:olmo_ruler_detailed}
    \label{tab:olmo_ruler_selected} 
    \small
    \setlength{\tabcolsep}{3pt}
    \begin{tabular}{@{}lrrr@{\hspace{1em}}rrr@{}}
        \toprule
        & \multicolumn{3}{c}{\textbf{4K}} & \multicolumn{3}{c}{\textbf{8K}} \\
        \cmidrule(lr){2-4}\cmidrule(lr){5-7}
        \textbf{Subtask} & \textbf{LoRA-only} & \textbf{\shortstack{GAPE+\\LoRA}} & $\Delta$ & \textbf{LoRA-only} & \textbf{\shortstack{GAPE+\\LoRA}} & $\Delta$ \\
        \midrule
        NIAH S1 & 100.00 & 100.00 & $0.00$ & 100.00 & 100.00 & $0.00$ \\
        NIAH S2 & 100.00 & 100.00 & $0.00$ & 100.00 & 100.00 & $0.00$ \\
        NIAH S3 & 100.00 & 100.00 & $0.00$ & 100.00 & 100.00 & $0.00$ \\
        NIAH MK1 & 99.00 & 99.00 & $0.00$ & 95.00 & 95.00 & $0.00$ \\
        NIAH MK2 & 99.00 & 98.00 & $-1.00$ & 100.00 & 98.00 & $-2.00$ \\
        NIAH MK3 & 100.00 & 99.00 & $-1.00$ & 90.00 & 89.00 & $-1.00$ \\
        Multi-value NIAH & 99.00 & 99.25 & $+0.25$ & 95.25 & 95.25 & $0.00$ \\
        Multi-query NIAH & 99.00 & 98.75 & $-0.25$ & 93.50 & 95.00 & $+1.50$ \\
        \addlinespace[3pt]
        Variable tracking & 88.80 & 87.80 & $-1.00$ & 22.60 & 13.40 & $-9.20$ \\
        \addlinespace[3pt]
        Common-word extraction & 94.30 & 93.40 & $-0.90$ & 43.40 & 45.30 & $+1.90$ \\
        Frequent-word extraction & 91.33 & 91.00 & $-0.33$ & 76.00 & 76.33 & $+0.33$ \\
        \addlinespace[3pt]
        QA-1 & 88.00 & 89.00 & $+1.00$ & 78.00 & 77.00 & $-1.00$ \\
        QA-2 & 62.00 & 63.00 & $+1.00$ & 53.00 & 52.00 & $-1.00$ \\
        \midrule
        \textbf{Mean (13 subtasks)} & 93.88 & 93.71 & $-0.17$ & 80.52 & 79.71 & $-0.81$ \\
        \addlinespace[8pt]
        \midrule
        & \multicolumn{3}{c}{\textbf{16K}} & \multicolumn{3}{c}{\textbf{32K}} \\
        \cmidrule(lr){2-4}\cmidrule(lr){5-7}
        \textbf{Subtask} & \textbf{LoRA-only} & \textbf{\shortstack{GAPE+\\LoRA}} & $\Delta$ & \textbf{LoRA-only} & \textbf{\shortstack{GAPE+\\LoRA}} & $\Delta$ \\
        \midrule
        NIAH S1 & 94.00 & 95.00 & $+1.00$ & 18.00 & 18.00 & $0.00$ \\
        NIAH S2 & 82.00 & 75.00 & $-7.00$ & 8.00 & 11.00 & $+3.00$ \\
        NIAH S3 & 88.00 & 83.00 & $-5.00$ & 7.00 & 7.00 & $0.00$ \\
        NIAH MK1 & 68.00 & 66.00 & $-2.00$ & 15.00 & 14.00 & $-1.00$ \\
        NIAH MK2 & 69.00 & 63.00 & $-6.00$ & 2.00 & 3.00 & $+1.00$ \\
        NIAH MK3 & 38.00 & 37.00 & $-1.00$ & 0.00 & 0.00 & $0.00$ \\
        Multi-value NIAH & 56.00 & 57.75 & $+1.75$ & 13.00 & 15.00 & $+2.00$ \\
        Multi-query NIAH & 49.75 & 48.75 & $-1.00$ & 9.25 & 11.75 & $+2.50$ \\
        \addlinespace[3pt]
        Variable tracking & 15.20 & 35.20 & $+20.00$ & 1.20 & 2.00 & $+0.80$ \\
        \addlinespace[3pt]
        Common-word extraction & 29.60 & 35.10 & $+5.50$ & 8.00 & 3.50 & $-4.50$ \\
        Frequent-word extraction & 69.67 & 68.33 & $-1.33$ & 32.67 & 44.00 & $+11.33$ \\
        \addlinespace[3pt]
        QA-1 & 65.00 & 61.00 & $-4.00$ & 14.00 & 15.00 & $+1.00$ \\
        QA-2 & 43.00 & 45.00 & $+2.00$ & 23.00 & 27.00 & $+4.00$ \\
        \midrule
        \textbf{Mean (13 subtasks)} & 59.02 & 59.24 & $+0.22$ & 11.62 & 13.17 & $+1.55$ \\
        \bottomrule
    \end{tabular}
\end{table}

\subsection{Ablation Study of the GAPE Mask Components}
\label{app:ablation}
We investigate the contribution of the individual components of the GAPE mask through an ablation study. Specifically, we train two additional models using the same 44M nanoGPT setup described in Appendix~B.1 and evaluate them on the RULER subsets considered here. In each variant, we ablate either the $l$ or the $g$ component of GAPE by making it non-learnable and fixing it to the corresponding value reported in Table~\ref{tab:train_conf_app}. Table~\ref{tab:ablation_comparison} reports the relative percentage difference in NLL with respect to the full GAPE model across the different subsets of the RULER benchmark. Negative values indicate improved performance for the ablated model relative to full GAPE (lower NLL), whereas positive values indicate degraded performance (higher NLL).

\begin{table*}[t]
\centering

\caption{Relative percentage difference in NLL (\%) between the ablated models and the full GAPE model across RULER subsets.
Negative values indicate improved performance of the ablated model relative to full GAPE (lower NLL), whereas positive values indicate degraded performance (higher NLL).}
\label{tab:ablation_comparison}
\small
\setlength{\tabcolsep}{6pt}
\begin{tabular}{lcccccccc}
\toprule
\textbf{Model} &
\textbf{cwe\_4k} &
\textbf{cwe\_8k} &
\textbf{niah\_4k} &
\textbf{niah\_8k} &
\textbf{qa\_4k} &
\textbf{qa\_8k} &
\textbf{vt\_4k} &
\textbf{vt\_8k} \\
\midrule

\texttt{g removed}
& \cellcolor{red!15} -3.10\%
& \cellcolor{red!15} -2.07\%
& \cellcolor{green!15} +0.63\%
& \cellcolor{green!15} +1.88\%
& \cellcolor{green!15} +2.94\%
& \cellcolor{green!15} +2.03\%
& \cellcolor{green!15} +4.54\%
& \cellcolor{green!15} +4.70\% \\

\texttt{l removed}
& \cellcolor{green!15} +0.62\%
& \cellcolor{red!15} -3.10\%
& \cellcolor{green!15} +3.34\%
& \cellcolor{green!15} +0.61\%
& \cellcolor{green!15} +4.73\%
& \cellcolor{green!15} +5.49\%
& \cellcolor{red!15} -0.02\%
& \cellcolor{green!15} +1.28\% \\

\bottomrule
\end{tabular}
\vspace{-0.15cm}
\end{table*}

Removing either gate degrades performance in 6 of the 8 evaluated task--length settings, with the clearest exceptions concentrated in Common Words Extraction. This suggests that the joint gating mechanism is most useful when the task requires selective positional routing, rather than providing a uniform benefit across all long-context tasks.

\subsection{Needle-In-A-Haystack Synthetic Experiment}
\label{sec:niah-settings}
We evaluate long-context retrieval in a controlled NIAH task. The vocabulary contains $27$ tokens: $10$ digit tokens, the special tokens \texttt{KEY}, \texttt{=}, and \texttt{?}, and $14$ filler tokens. Each input sequence has length $L$ and contains $n$ needles of the form $\texttt{KEY} = d_k$, where $d_k \in \{0,\ldots,9\}$ is the digit associated with the $k$-th
needle. An input sequence is a string of the form \textit{``[filler...]KEY=$d_1$[filler...]...KEY=$d_n$?''}. The remaining positions are filled with random filler tokens, and the last token is always the query token \texttt{?}. The target label is defined only at the final position and corresponds to the digit of a selected needle. Needles are placed using a chunk-based procedure. The usable prefix $\{0,\ldots,L-2\}$ is divided into $N$ approximately equal chunks, and one needle is sampled uniformly inside each chunk. This prevents overlap while ensuring that needles appear at different depths throughout the context. Unless otherwise specified, the number of needles scales with sequence length as: $n = \lfloor L/64 \rfloor$. Thus, the training setting uses $L=2048$ and $n=32$ needles. At test time, we
evaluate length extrapolation on: $L \in \{2048,4096,8192,16384\}$, corresponding to $1\times$, $2\times$, $4\times$ and $8\times$ the training context length, with the number of needles recomputed using the same rule.

\begin{figure}[t]
    \centering
    \includegraphics[width=0.75\linewidth]{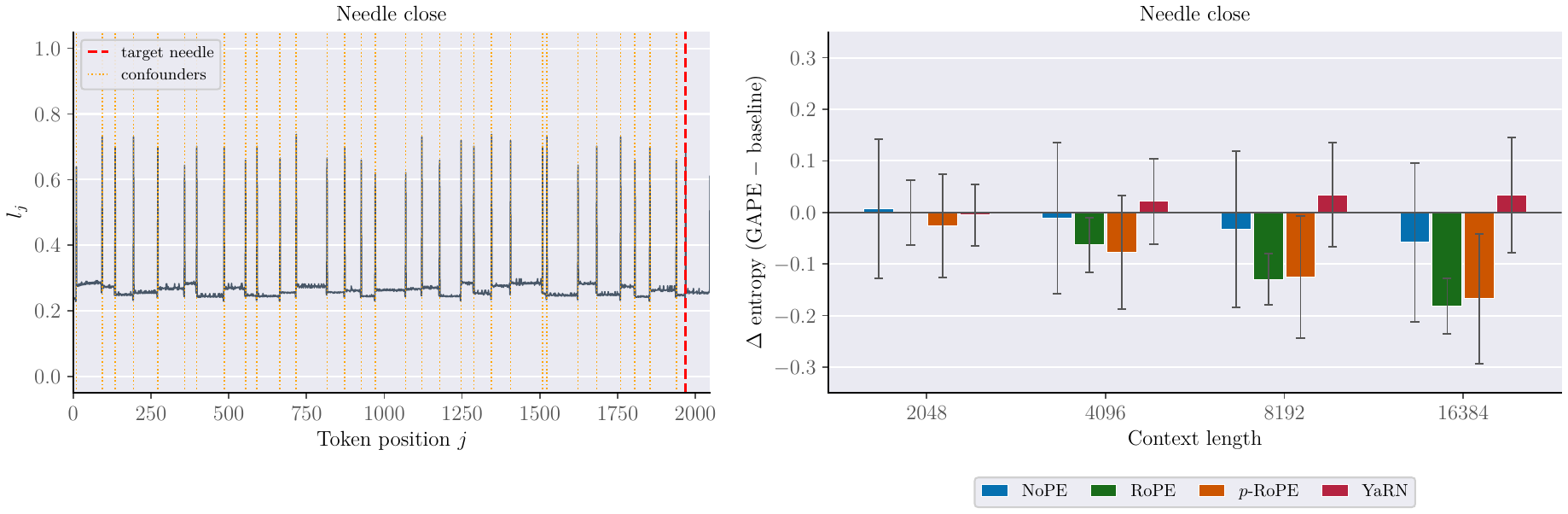}
    \caption{
    \emph{Mechanistic behavior of the GAPE gates in the NIAH task.}
    \textbf{Left:} \emph{Landmark gate $l_j$ at layer 2, head 1.}
    The gate produces sparse peaks at the needle positions, assigning them stronger protection and accessible despite surrounding filler and distractor tokens.
    \textbf{Right:} \emph{Average difference in \([0,1]\)-normalized attention entropy between GAPE and the corresponding base positional encoding, averaged over layers and heads.}
    Lower entropy indicates sharper and more concentrated attention. GAPE generally reduces entropy relative to the corresponding base positional encoding, especially for RoPE and $p$-RoPE, suggesting that the learned mask contracts irrelevant background context while preserving relevant tokens. Means and standard deviations over 10 samples.
    }
    \label{fig:GAPE_gate_analysis}
\end{figure}


\paragraph{Retrieval Regimes and Model Settings.}
We consider two retrieval regimes. In \texttt{needle-first} (needle far), the model must return the digit of the first needle, while in \texttt{needle-last} (needle close), it must return the digit of the last needle. These regimes isolate different positional demands: \texttt{needle-first} requires preserving distant early information, whereas \texttt{needle-last} favors recent information near the query.

All models are decoder-only, GPT-2-style Transformers built from Karpathy's \href{https://github.com/karpathy/nanoGPT}{nanoGPT}, with $2$ layers, $2$ attention heads per layer, embedding dimension $64$, and dropout disabled. We also include a NoPE+learnable-temperature baseline, which augments NoPE with one positive, input-independent temperature $\tau_{\ell,h}=\exp(\theta_{\ell,h})$ per layer $\ell$ and head $h$. Attention logits are divided by $\tau_{\ell,h}$ in addition to the standard $\sqrt{d_h}$ scaling. Temperatures are initialized to $1$ and optimized jointly with the model, adding only four scalar parameters and no positional encoding or GAPE mask. This baseline tests whether learned attention rescaling alone can explain GAPE's gains.

We train for $5000$ steps with batch size $64$, AdamW optimization, learning rate $3 \times 10^{-4}$, minimum cosine-decayed learning rate $3 \times 10^{-5}$, and weight decay $0.1$. Validation is performed every $100$ steps on $1000$ examples, with early stopping disabled for the matched comparisons. During evaluation, accuracy is computed from the prediction at the final query position over the $10$ digit tokens.

\paragraph{Attention Entropy in the Needle-in-a-Haystack Setting.}
\label{app:niah_entropy}

We further analyze GAPE through attention entropy in the same controlled NIAH setting. For a query position $i$ and causal attention distribution $\alpha_i$ over keys $j \leq i$, we define the Shannon entropy as:
\begin{equation}
    H(\alpha_i)
    =
    -\sum_{j \leq i} \alpha_{i,j}\log \alpha_{i,j}
\end{equation}
We consider normalized entropy $H^\text{norm}(\alpha_i)=\frac {H(\alpha_i)}{\log i}\in[0,1]$, since $\alpha_i$'s support has size $i$, to ease its interpretation.
Figure~\ref{fig:GAPE_gate_analysis} reports the average normalized-entropy difference $\Delta H$ over heads and layers for each context length.
\begin{equation}
     \Delta H = H^\text{norm}_{\text{PE+GAPE}} - H^{\text{norm}}_{\text{PE}}
\end{equation}
with respect to the corresponding base positional encoding. Negative values therefore indicate sharper attention, while positive values indicate more diffuse attention. 

In the \textsc{Needle-Close} setting, GAPE lowers attention entropy for NoPE, RoPE, and $p$-RoPE, with the effect generally strengthening as context length grows, consistent with GAPE's suppression of irrelevant background context. YaRN+GAPE follows a different entropy pattern, potentially reflecting the fact that YaRN already modifies the rotary spectrum to reduce extrapolation instability. However, attention entropy is not itself predictive of retrieval quality: YaRN+GAPE achieves better retrieval performance than YaRN despite not exhibiting the same entropy reduction.



\paragraph{Query-gate Adaptation Across Retrieval Regimes.}
\label{app:query-gate_behaviour}
Figure~\ref{fig:niah_gi_close_far} provides a direct view of how the learned query
gate adapts to the retrieval regime in the synthetic NIAH task. When the target
needle is close to the query, both attention heads assign a large value to $g_i$ at
the final query position, which corresponds to a stronger contraction of unprotected
tokens and therefore a sharper recency bias. In contrast, when the target needle is
far from the query, the learned gate remains close to zero for both heads, effectively
relaxing the contraction and allowing distant tokens to remain accessible. This
behavior is indeed consistent with the intended role of GAPE: $g_i$ does not impose a fixed distance penalty, but instead adapts the effective context length to the demands of
the query. In this sense, the model learns to strengthen recency when the answer is
local and to suppress it when successful retrieval requires preserving long-range
evidence.
\begin{figure}[t]
\centering
\includegraphics[width=0.5\linewidth]{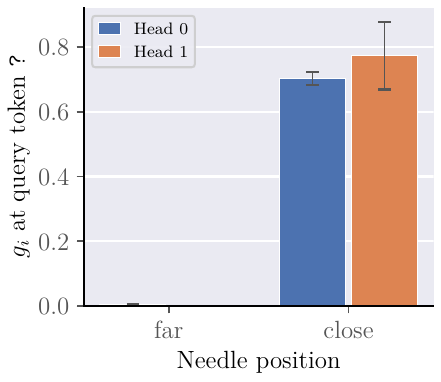}
\caption{
\emph{Query-gate behavior in the NIAH setting.}
Average query-gate value $g_i$ at the final query token for the two attention heads,
reported separately for the \textsc{Needle-Far} and \textsc{Needle-Close} regimes.
When the target needle is close to the query, both heads learn substantially larger
$g_i$ values, corresponding to stronger contraction of unprotected context. When the
target is far, $g_i$ remains near zero, indicating that the model relaxes the recency
bias in order to preserve access to distant relevant information.
}
\label{fig:niah_gi_close_far}
\end{figure}

\subsection{Frequency-Channel Norm Distribution}
\label{app:qk_norm_analysis}

To better understand how GAPE changes the internal use of rotary dimensions, we analyze the distribution of query/key norms across frequency channels. For each positional encoding, we group the query/key dimensions according to their associated rotary frequency channel and compute the average norm over a fixed reference batch. We compare GAPE against RoPE and $p$-RoPE in \autoref{fig:qk_norms_by_channel_app}.
\begin{figure}[t]
    \centering
    \includegraphics[width=0.95\linewidth]{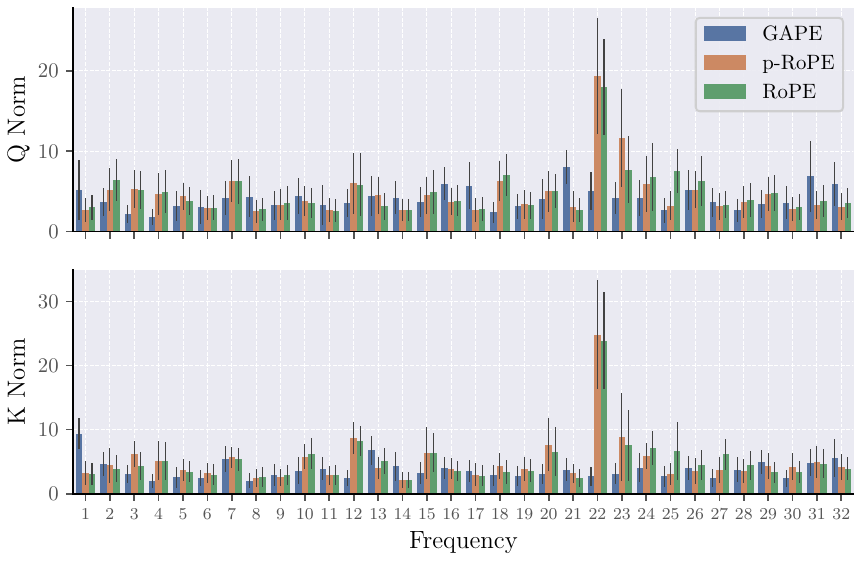}
    \caption{
    \emph{Query/key norm distribution by frequency.}
    We report the average norm value across rotary frequency channels for GAPE, $p$-RoPE, and RoPE. RoPE and $p$-RoPE develop a pronounced spike around a narrow frequency band, while GAPE produces a more distributed norm profile. Error bars indicate variation across the measured samples.
    }
    \label{fig:qk_norms_by_channel_app}
\end{figure}
The results show that RoPE and $p$-RoPE allocate a large norm to a small subset of frequency channels. This behavior is consistent with the interpretation that purely multiplicative positional encodings rely on the rotary feature space for both local positional matching and long-range interference control. In particular, the observed norm concentration may reflect a compensatory mechanism in which selected rotary dimensions are amplified to sharpen positional discrimination. GAPE exhibits a different pattern. Because it introduces an explicit structural bias $M_{i,j}$, long-range suppression can also be handled through the query gate $g_i$ and landmark gate $l_j$, rather than solely through phase-modulated semantic channels. The flatter norm profile observed under GAPE is therefore consistent with a decomposition in which rotary dimensions preserve local phase-based matching while the structural pathway contributes to long-range suppression and landmark preservation. Taken together with the attention and routing analyses, these results are consistent with GAPE separating local positional geometry from long-range structural filtering, while avoiding strong concentration of query/key norms in a narrow rotary frequency band.

\subsection{Full Routing Dynamics of GAPE and Head Specialisation}
\label{app:GAPE_routing_dynamics}
This section provides an extended analysis of the routing behavior summarized in Figure~\ref{fig:GAPE_final_layer_analysis}. We complement the downstream evaluation with a mechanistic analysis of the learned GAPE routing variables. The purpose of this analysis is to verify whether the additional structural pathway is actively used by the model, and whether its learned dynamics match the intended decomposition between suppression and preservation. 
\autoref{fig:GAPE_routing_dynamics} reports the evolution of the induced structural mask ${M}_h$, landmark activations $\bar{l}_h$, query gates $\bar{g}_h$, and head amplitudes $\Gamma_h$, averaged over a fixed reference batch and shown separately for each layer and attention head.
\begin{figure}[htb]
    \centering

    \begin{subfigure}[t]{\linewidth}
        \centering
        \includegraphics[width=\linewidth]{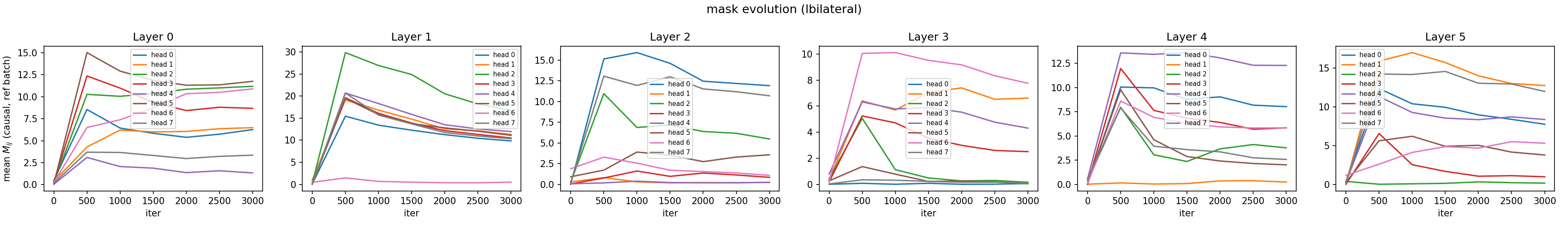}
        \caption{Mask evolution.}
        \label{fig:GAPE_mask_evolution}
    \end{subfigure}

    \vspace{0.3em}

    \begin{subfigure}[t]{\linewidth}
        \centering
        \includegraphics[width=\linewidth]{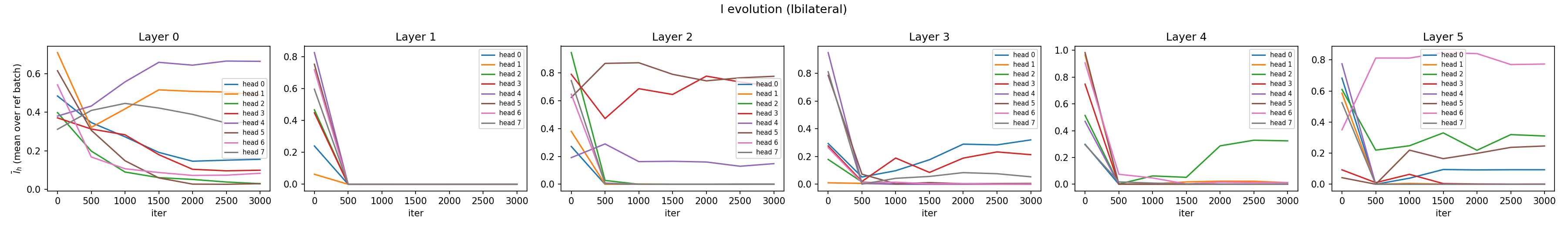}
        \caption{Landmark activation evolution.}
        \label{fig:GAPE_l_evolution}
    \end{subfigure}

    \vspace{0.3em}

    \begin{subfigure}[t]{\linewidth}
        \centering
        \includegraphics[width=\linewidth]{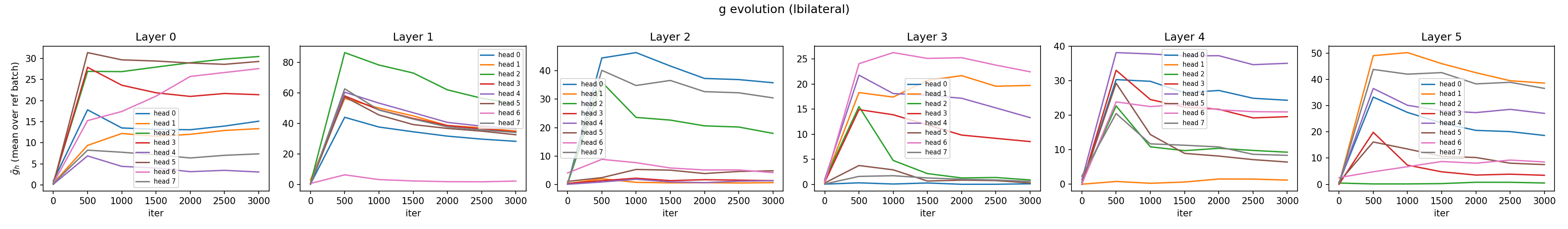}
        \caption{Query-gate evolution.}
        \label{fig:GAPE_g_evolution}
    \end{subfigure}

    \vspace{0.3em}

    \begin{subfigure}[t]{\linewidth}
        \centering
        \includegraphics[width=\linewidth]{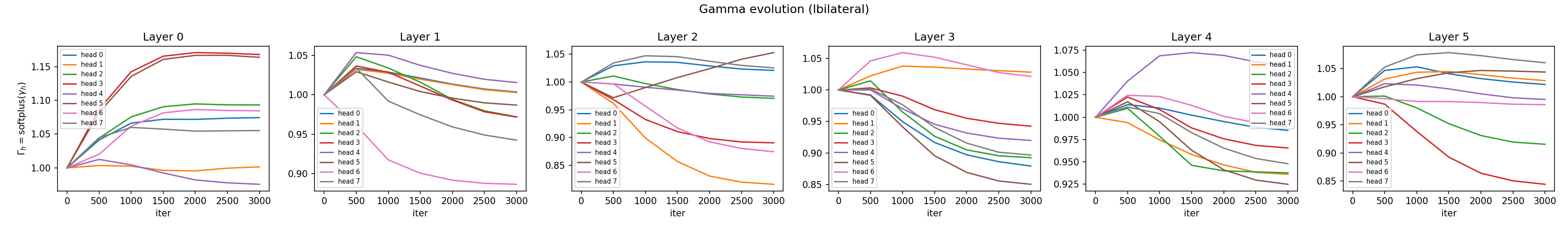}
        \caption{Head amplitude evolution.}
        \label{fig:GAPE_gamma_evolution}
    \end{subfigure}

    \caption{
    \textbf{Evolution of GAPE routing variables during training.}
    We track the mean structural mask $\bar{M}_h$, landmark activation $\bar{l}_h$, query gate $\bar{g}_h$, and head amplitude $\Gamma_h$ across layers and heads. GAPE learns a heterogeneous routing structure: several heads develop high mask values and act as contractive filters, while a smaller subset preserves non-trivial landmark activations, most prominently in Layers~2 and~5. The dynamics indicate that GAPE does not implement a uniform distance-decay prior, but instead separates background suppression from landmark-mediated preservation.
    }
    \label{fig:GAPE_routing_dynamics}
\end{figure}

\paragraph{Mask Evolution.}
The structural mask is the most direct observable of the GAPE mechanism, since it is injected into the pre-softmax attention logits. As shown in \autoref{fig:GAPE_mask_evolution}, the learned masks consistently move toward higher values across layers, indicating that the structural pathway is active and directly reshapes the attention distribution. The suppression pattern is highly non-uniform across heads and depth: early layers already separate into mildly biased and strongly contractive heads, with Layer~1 showing the most coherent shift. Intermediate layers exhibit a mixed regime, where some heads remain near-neutral while others become strongly suppressive. The final layer preserves this separation, suggesting that GAPE maintains both filtering and preservation channels close to the output. From the image, we can observe that the model learns head-dependent contraction patterns, allowing different heads to specialize in different degrees of structural suppression.

\paragraph{Landmark-score Evolution.}
Landmark activations reveal how GAPE avoids collapsing into a uniform decay mechanism. Since $l_j$ controls whether a key is protected from structural suppression, the trajectories of $\bar{l}_h$ identify which heads retain landmark-mediated routing capacity. As shown in \autoref{fig:GAPE_l_evolution}, landmark protection is sparse and strongly layer-dependent. Layer~0 gradually reduces landmark activations in most heads, while retaining at least one comparatively stable channel. Layer~1 shows an almost complete collapse of $\bar{l}_h$ across heads, matching the strong mask activation in \autoref{fig:GAPE_mask_evolution} and suggesting that this layer mainly treats tokens as unprotected. By contrast, Layer~2 maintains or increases landmark activations in several heads, some approaching high values, while Layers~3 and~4 return to a mostly suppressive regime, even if the latter shows partial recovery in a small subset of heads. Layer~5 exhibits the clearest late-stage specialization, with some heads remaining near zero and others preserving high landmark activations throughout training. The model is indeed showing a selective-routing interpretation, learning an intermediate regime in which most heads contribute to suppression, while a smaller subset preserves protected-token channels.

\paragraph{Query-gate Evolution.}
The query gate $g_i$ controls the strength of contraction induced by the current query. It does not determine which keys are important, but rather, it sets how aggressively weakly protected keys are attenuated. For this reason, $\bar{g}_h$ must be interpreted jointly with the landmark activations $\bar{l}_h$ and the resulting masks $\bar{M}_h$. The trajectories show that $g_i$ is actively used as a routing variable. Heads with large $\bar{g}_h$ and collapsed $\bar{l}_h$ enter a strongly contractive regime, assigning substantial recency bias to most keys and reducing the effective context available to the query. This accounts for the large mask values observed in suppressive heads. Conversely, heads with large $\bar{g}_h$ but non-negligible $\bar{l}_h$ implement the characteristic GAPE behavior: the structural filter becomes sharp, while landmark tokens remain protected from the full penalty. Thus, a large query gate alone corresponds to strong recency-biased contraction, whereas a large query gate combined with nonzero landmark activations yields selective contraction. The layer-wise dynamics suggest that Layers~1 and~3 primarily use $g_i$ for filtering, while Layers~2 and~5 provide the clearest evidence of landmark-conditioned preservation.

\paragraph{Head-amplitude Evolution.}
The amplitude $\Gamma_h$ sets the head-wise scale of the mask. $\Gamma_h$ evolves more smoothly and remains comparatively bounded, indicating that it mainly acts as a calibration parameter, modulating how strongly each head participates in the structural pathway. Since strong mask values emerge without an uncontrolled growth of $\Gamma_h$, the observed suppression is better attributed to structured routing: heads with low landmark protection and active query gates become contractive, whereas heads with sustained landmark activations retain protected channels.
\begin{figure}[htb]
    \centering
    \includegraphics[width=\linewidth]{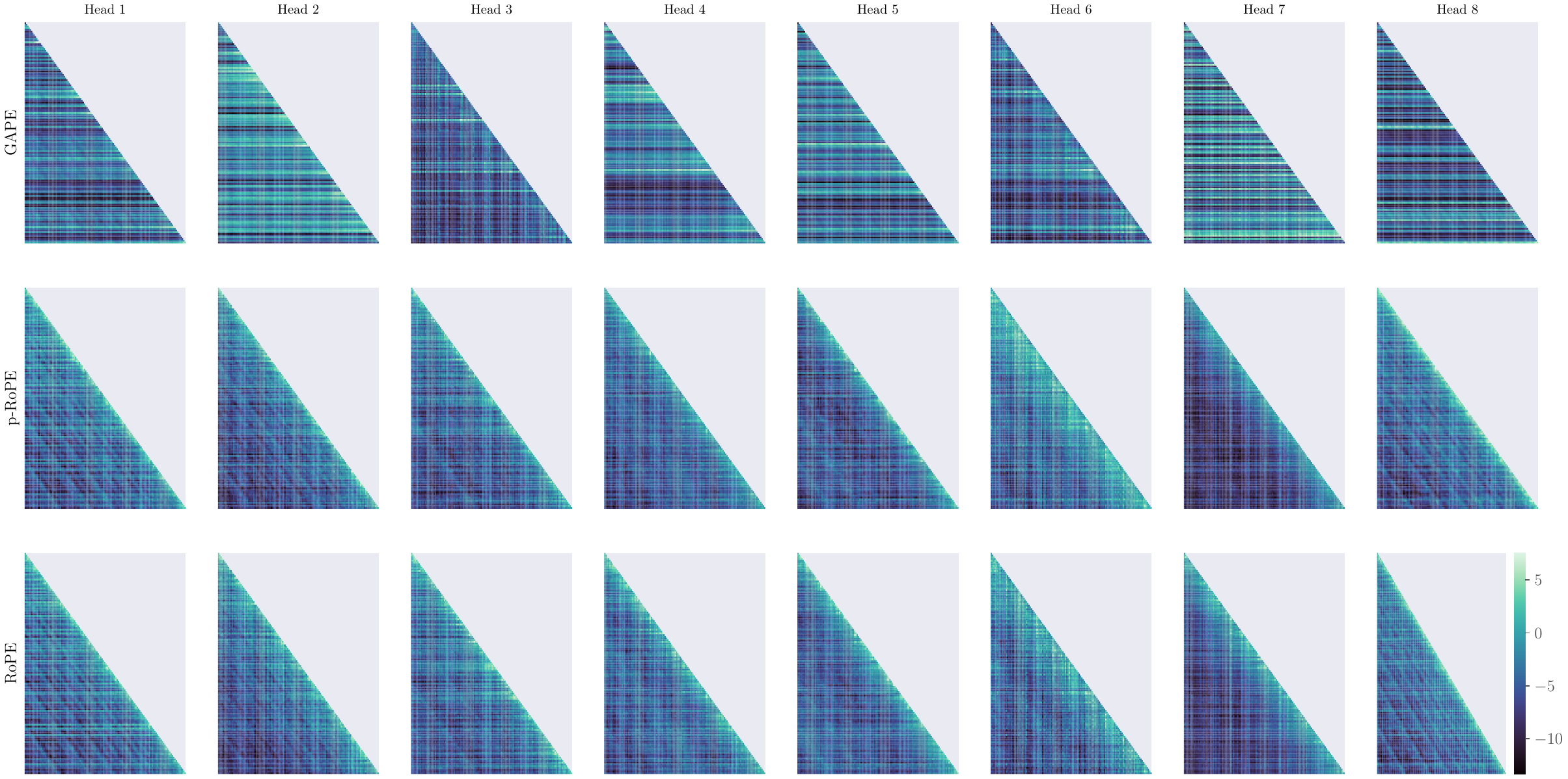}
    \caption{
    \emph{Layer~5 attention maps across heads.}
    We compare the realised attention maps of GAPE, $p$-RoPE, and RoPE across the eight heads of the final layer. RoPE and $p$-RoPE exhibit comparatively smoother causal patterns, whereas GAPE develops more heterogeneous head-wise structures, with some heads showing sharper contraction and others retaining broader or banded access. This provides a qualitative view of the learned head specialisation induced by the GAPE routing mechanism.
    }
    \label{fig:attention_maps_layer5_app}
\end{figure}

\paragraph{Head Specialization in Attention Maps.}
The routing-variable dynamics are directly reflected in the realized attention patterns. As shown in \autoref{fig:attention_maps_layer5_app}, GAPE produces more heterogeneous attention maps across heads than RoPE and $p$-RoPE. While the rotary baselines are largely governed by causal structure and phase-based similarity, GAPE introduces clear head-wise differentiation: some heads display sharp contraction patterns, consistent with strong mask activation and active query gates, whereas others retain broader or more structured access, consistent with non-trivial landmark protection. Importantly, this does not indicate a uniform shortening of the context window. Rather, it confirms that GAPE assigns different roles to different heads, separating background suppression from landmark-mediated preservation in a way that is visible both in the learned routing variables and in the final attention maps.

\subsection{Complexity Analysis}
\label{app:complexity}
We analyze the computational and inference-memory complexity of GAPE relative to
standard RoPE. Let $B$ denote the batch size, $T$ the sequence length, $L$ the number
of Transformer layers, $H$ the number of attention heads, $d$ the head dimension,
and $C = Hd$ the width of the model. We implement GAPE and RoPE through the same tensor shapes passed to scaled dot-product attention, namely $(B,H,T,d)$ for queries, keys, and values, such that the dominant attention computation is unchanged.

\paragraph{FLOPs.}
For a single Transformer layer, both RoPE and GAPE share the same costs across the QKV projection, causal attention, output projection, and MLP. Ignoring lower-order normalization terms, these costs are:
\[
\text{FLOPs}_{\mathrm{ROPE}}
=\text{FLOPs}_{\mathrm{GAPE}}=
2BHT^2d + 24BTC^2 + O(BTC)
\]
The term $2BHT^2d$ comes from the two products of the matrix in attention, $QK^\top$ and $\mathrm{Attn} \cdot V$, while the term $24BTC^2$ comes from the QKV projection, the output projection, and the two-layer MLP. The cost of the specific positional encoding is only linear in sequence length.
\[
\text{FLOPs}_{\mathrm{RoPE}}
=
\text{FLOPs}_{\mathrm{GAPE}}
=
O(BHT^2d + BTC^2)
\]

\paragraph{KV-cache Memory at Inference.}
During autoregressive decoding, each layer stores past keys and values in a KV
cache. For RoPE, the rotated key $k_t$ and value $v_t$ are appended to the cache.
For GAPE, the landmark value and positional information are already encoded in
the modified key tensor, so the modified key $k_t$ and the value $v_t$ are
cached. No additional gate values, landmark scores, or positional states need to
be stored. Thus, both RoPE and GAPE cache tensors of identical shape $K,V \in \mathbb{R}^{B \times H \times T_{\mathrm{ctx}} \times d}$. For each layer, where $T_{\mathrm{ctx}}$ is the current decoded context length. The total KV-cache size is therefore:
\[
M_{\mathrm{KV}}
=
2L \cdot B \cdot T_{\mathrm{ctx}} \cdot H \cdot d
=
2L \cdot B \cdot T_{\mathrm{ctx}} \cdot C
\quad \text{elements}
\]
In fp16 inference, this corresponds to:
\[
M_{\mathrm{KV}}^{\mathrm{fp16}}
=
4LBT_{\mathrm{ctx}}C
\quad \text{bytes}
\]
Therefore, GAPE does not introduce additional KV-cache memory overhead relative to
RoPE:
\[
M_{\mathrm{KV}}^{\mathrm{GAPE}}
=
M_{\mathrm{KV}}^{\mathrm{RoPE}}
\]
At each decoding step, the dominant per-layer cost is also unchanged. The model
computes projections for the new token, is compared with the cached keys, and
aggregates cached values:
\[
\text{Decode FLOPs per layer}
=
24C^2 + 4CT_{\mathrm{ctx}} + O(C)
\]
where the $O(C)$ term contains the positional-encoding computation. Since both
RoPE and GAPE have only linear positional overhead, the decode-time complexity is
asymptotically identical.

\end{document}